\documentclass[twoside]{article} \usepackage{aistats2017}

\usepackage[utf8]{inputenc} 
\usepackage[T1]{fontenc}    
\usepackage{hyperref}       
\usepackage{url}            
\usepackage{booktabs}       
\usepackage{amsfonts}       
\usepackage{nicefrac}       
\usepackage{microtype}      

\usepackage{qiangstyle}

\begin{document}

%

%
\twocolumn[
\aistatstitle{Black-Box Importance Sampling}
\aistatsauthor{ Qiang Liu \And Jason D. Lee }
\aistatsaddress{ Dartmouth College \And  University of South California } ]

%



%


\newcommand{\fix}{\marginpar{FIX}}
\newcommand{\new}{\marginpar{NEW}}

\begin{abstract}
Importance sampling is widely used in machine learning and statistics, but its power is limited by the restriction of using \emph{simple} proposals 
for which the importance weights can be tractably calculated.  
We address this problem by 
studying \emph{black-box importance sampling methods} 
that calculate importance weights for samples generated from any unknown proposal or black-box mechanism. 
Our method allows us to use better and richer proposals to solve difficult problems, 
and (somewhat counter-intuitively) also has the additional benefit of improving the estimation accuracy beyond typical importance sampling. Both theoretical and empirical analyses are provided. 
\end{abstract}

\section{Introduction}
Efficient Monte Carlo methods are workhorses for modern Bayesian statistics and machine learning. 
Importance sampling (IS) and Markov chain Monte Carlo (MCMC) are two fundamental tools widely used when
it is intractable to draw exact samples from the underlying distribution $p(x)$. 
IS uses an \emph{simple} proposal distribution $q(x)$ to draw a sample $\{x_i\}$, and attaches it with a set of importance weights that are proportional to the probability ratio $p(x_i)/q(x_i)$. 
MCMC methods, on the other hand, rely on simulating Markov chains whose equilibrium distribution matches the target distribution. 

Unfortunately, both importance sampling (IS) and MCMC have their own critical weaknesses. 
IS heavily relies on a good proposal $q(x)$ that closely matches the target distribution $p(x)$ to obtain accurate estimates. 
However, it is critically challenging, or even impossible, to design good proposals for high dimensional complex target distributions, given the restriction of using simple proposals. 
Therefore, alternative methods that do not require to calculating proposal probabilities would greatly enhance the powerful of IS, yielding 
efficient solutions for difficulty problems. 

On the other hand, MCMC approximates the target distribution with an (often complex) distribution simulated from a large number of steps of Markov transitions,
and has been widely used to solve complex problems. However, MCMC has a long-standing difficulty accessing its convergence, and one may get absurdly wrong results when using non-convergent results \citep[e.g.,][]{morris1996ising}. 
In addition, the computational cost of MCMC becomes critically expensive when the number of data instances is very large (a.k.a. the big data setting). 
A number of approximate versions of MCMC have been developed recently to deal with the big data issue \citep[e.g.,][]{welling2011bayesian, alquier2016noisy}, 
but these methods usually no longer converge to the correct stationary distribution. 
%


Motivated by combining the advantages of IS and MCMC, we study \emph{black-box importance sampling} methods that can calculate importance weights for any given sample $\{x_i\}_{i=1}^n$ generated from arbitrary, unknown black-box mechanisms. 
Such methods allow us to use highly complex proposals that closely match the target distribution, without worrying about the computational tractability of the typical importance weights. 

Interestingly, the black-box methods, despite using no information of the proposal distribution, 
can actually give better estimation accuracy than the typical importance sampling that leverages the proposal information. 
This appears to be a paradox (using less information yet getting better results), 
but is consistent with the arguments of \citet{o1987monte} 
that ``Monte Carlo (that uses the proposal information)
is fundamentally unsound'' for violating the Likelihood Principle, 
and the interesting results of \citet{henmi2007importance, delyon2014integral}
that certain types of approximate versions of IS weights reduce the variance over exact IS weights. 
As an example of application, we apply black-box importance weights to samples simulated by a number of short Markov chains, 
in which MCMC helps provide a complex proposal that are ``crudely'' closely to the target distribution, and the black-box weights further refine the result. 
In this way, we obtain consistent estimators even from un-convergent MCMC results, or approximate MCMC transitions that appear commonly in big data settings. 

Beyond MCMC, 
black-box IS can be used to refine many other approximation methods 
related to complex generation mechanisms, 
including variational inference with complex proposals \citep[e.g.,][]{rezende2015variational}, 
  bootstrapping \citep{efron2012bayesian} 
 and perturb-and-MAP methods \citep{hazan2013sampling, papandreou2011perturb}.  
Further, we envision our method can find more applications in many areas where importance sampling or variance reduction plays an importance role, 
such as 
probabilistic inference in graphical model \citep[e.g.,][]{liu2015probabilistic}, 
variance reduction for variational inference \citep[e.g.,][]{wang2013variance, ranganath2014black}
and policy gradient \citep[e.g.,][]{greensmith2004variance}, 
 covariance shift in transfer learning \citep[e.g.,][]{sugiyama2008direct}, off-policy reinforcement learning \citep[e.g.,][]{li2015toward}, and stochastic optimization \citep[e.g.,][]{zhao2014stochastic}.


Our black-box importance weights are calculated by a convex quadratic optimization, based on  
minimizing a recently proposed kernelized Stein discrepancy that measures the goodness of a sample fit with an unnormalized distribution 
\citep{oates2014control, chwialkowski2016kernel, liu2016kernelized}; 
this makes our method widely applicable for unnormalized distributions that widely appear in Bayesian statistics and machine learning.

\subsection*{Related Works}\label{sec:related}
 Our method is closely related to \citet{briol2015probabilistic, briol2015frank, oates2014control},
 which combine Stein's identity with Bayesian Monte Carlo  \citep{o1991bayes, ghahramani2002bayesian} 
 and control variates, respectively, 
 and can also be interpreted as a form of importance weights similar to our method. 
 The key difference is that the weights in their method can be negative and are not normalized to sum to one, 
while our approach explicitly optimizes the weights in the probability simplex, 
which helps provide more stable practical results as we illustrate both theoretically and empirically in our work. 
We provide a more throughout discussion in Section~\ref{sec:other}. 

An alternative approach for black-box weights is to directly approximate the underlying proposal distribution $q$ 
with an estimator $\hat q$ and use the corresponding ratio 
$p(x)/\hat q(x)$ as the importance weight. 
\citet{henmi2007importance, delyon2014integral} showed that certain types of approximation $\hat q$ can improve, rather than deteriorate, 
the performance compared with the weight with the exact $q$. 
However, the method by \citet{henmi2007importance} 
is not widely applicable since it requires to solve a maximum likelihood estimator in a parametric family that include the proposal distribution; 
\citet{delyon2014integral} uses a kernel density estimator for $q$ and tends to give unstable empirical results as we show in our experiments. 
Related to this, there is a literature 
in semi-supervised learning for covariance shifts \citep[e.g.,][]{sugiyama2012machine}
that estimates the density ratio $p(x)/q(x)$ given two samples $\{x_i\}\sim p$ and $\{y_i\} \sim q$, when both $p$ and $q$ are unknown. 

There are also other directions where the advantages of IS and MCMC can be combined, 
including adaptive importance sampling  \citep[e.g.,][to only name a few]{martino2015layered, botev2013markov, beaujean2013initializing, yuan2013novel}, and sequential Monte Carlo \citep[e.g.,][]{smith2013sequential, robert2013monte, neal2001annealed}. The black-box techniques 
can be combined with these methods to obtain more powerful, adaptive methods. 

%
\todo{This approach ``fixes" MCMC, by adjusting the samples from MCMC with proper importance weights, and is orthogonal to the line of work that diagnostic the weights. }
%
%
\todo{  We studied the convergence rate of our method in cases when $\{x_i\}_{i=1}^n$ is i.i.d. drawn from a distribution; it in fact has super-root-$n$ rate under proper smoothness conditions. }
\todo{We also make connections with two lines of previously unconnected works: Bayesian Monte Carlo \& Control variate methods \citep{oates2014control, oates2016convergence}, and estimated importance sampling \citep{henmi2007importance, delyon2014integral}. }

\todo{
There are a number of existing methods that can also be viewed providing black-box importance weights, 
including Bayesian Monte Carlo \citep{o1991bayes, ghahramani2002bayesian} and the related (linear) control variates method \citep[e.g.,][]{liu2008monte}, 
 as well as an approximate importance sampling method 
 \citep{henmi2007importance, delyon2014integral} 
 that can be interpreted as a \emph{multiplicative} control variates method \citep{nelson1987control}. 
 By viewing these methods as black-box importance weights (which seems unusual in the literature), 
 can also open up more applications for these methods. 
 We also provide a unified view of these methods. 
 }
 
\todo{
Our method draws a particular close connection with 
Bayesian Monte Carlo \citep{o1991bayes, ghahramani2002bayesian, briol2015probabilistic}, 
and control variates  \citep[e.g.,][]{liu2008monte} and control functional \citep{oates2014control}, 
which can be viewed as providing different forms of importance weights that also 
work for black-box proposals. 
The disadvantage of the typical Bayesian Monte Carlo, however, is that it requires to evaluate 
a kernel expectation term which is computationally tractable only for simple distributions with typical kernels. 
\citep{briol2015probabilistic, oates2014control} 
proposed to use Stein kernel 
in Bayesian Monte Carlo and control variates to allow tractable computation of the kernel expectation.
We will discuss futher connections with \citep{briol2015probabilistic, oates2014control} in Section~\ref{sec:other}. 
}

\paragraph{Preliminary and Notation}
Let $\k(x,x') \colon \X \times \X \to \R$ be a positive definite kernel; we denote by $\k(x, \cdot)$ the one-variable function for each fixed $x$. 
The reproducing kernel Hilbert space (RKHS) $\H$ of $\k(x,x')$ is the closure of linear span $\{f \colon  f= \sum_{i=1}^m a_i \k(x, x_i), ~~~ a_i \in \R, ~~ m\in \mathbb{N}, ~ x_i \in \X  \}$, equipped with an inner product $\la f, ~g \ra_{\H}= \sum_{ij} a_i b_j \k(x_i,x_j)$ for 
$f= \sum_{i=1}^m a_i \k(x, x_i)$ and 
$g  = \sum_i b_i \k(x,x_i)$. One can verify that such $\H$ has a \emph{reproducing} property in that $f = \la f, ~ \k(x, \cdot)\ra_{\H}$. 
We use $\Op(\cdot)$ for the Big O in probability notation.

\section{Background: Kernelized Stein Discrepancy}
\label{sec:background}

We give a brief introduction to Stein's identity and 
kernelized Stein discrepancy (KSD) \citep{liu2016kernelized, oates2014control, chwialkowski2016kernel} which forms the foundation of our method. 

Let $p(x)$ be a continuously differentiable (also called smooth) density supported on $\X\subseteq \RR^d$. 
We say that a smooth function $f(x)$ is in the Stein class of $p(x)$ if 
\begin{align}
\int_{\X}  \nabla_x (p(x)f(x)) dx = 0,
\label{equ:steinclass}
\end{align}
which can be implied by a zero boundary condition $p(x)f(x) = 0$, $\forall x \in \partial \X$ when $\X$ is compact, or $\lim_{||x||\to \infty} f(x)p(x) = 0 $ when $\X = \R^d$. 
For $f(x)$ in the Stein class of $p(x)$, 
Stein's identity shows  
\begin{align}
\label{def:steq1}
\begin{split}
&  \E_{x\sim p}[\score_p(x) f(x)  + \nabla_x f(x)] = 0,   \\
&\text{where} ~~~~~
\score_p(x) = \nabla_x \log p(x),
\end{split}
\end{align}
which is in fact a direct rewrite of \eqref{equ:steinclass} using the product rule of derivatives. We call $\score_p(x) \coloneqq \nabla_x \log p(x)$ the score function of $p(x)$. 
Note that calculating $\score_p(x)$ does not depend on the normalization constant in $p(x)$, that is, when $p(x) = f(x)/Z$
where $Z$ is the normalization constant and is often critical difficult to calculate, 
we have $\score_p(x) = \nabla_x \log f(x)$, independent of $Z$. 
This property makes Stein's identity a powerful practical tool for handling unnormalized distributions that widely appear in  machine learning and statistics. 

We can ``kernelize''  Stein's identity with a smooth positive definite kernel $\k(x,x')$ for which 
$\k(x, x')$ is in the Stein class of $p(x)$ for each fixed $x'\in \X$ (we say such $\k(x,x')$ is in the Stein class of $p$ in this case). 
By first applying \eqref{def:steq1} on $k(x,x')$ with fixed $x'$ and subsequently with fixed $x$, we can get the following kernelized versions of Stein's identity: 
\begin{align}
\E_{x \sim p} \big [  \kp(x,x') \big] = 0,  ~~~~ \forall x' \in \X, 
\label{equ:Dhh}
\end{align}
where $x,x'$ are i.i.d. drawn from $p$ and $\kp(x,x')$ is a new kernel function defined via
\begin{align}
& \kp(x,x')  = \score_p(x)^\top \k(x,x') \score_p(x')  +   \score_p(x)^\top \nabla_{x'} \k(x,x')  \notag \\ 
& ~~~~~ +  \score_p(x')^\top \nabla_{x} \k(x,x') + \trace(\nabla_{x,x'} \k(x,x')). \label{equ:uq}
\end{align}
See Theorem 3.5 of \citet{liu2016kernelized}. 
\eqref{equ:Dhh} suggests that $p(x)$ is an eigenfunction of kernel $\kp(x,x')$ with zero eigenvalue. 
In fact, let $\Hp$ be the RKHS related to $\kp(x,x')$, then all the functions $h(x)$ in $\Hp$ are orthogonal to $p(x)$ in that $\E_p[h(x)]= 0$. 
Such $\Hp$ were first studied in \citet{oates2014control}, in which it was used to 
define an infinite dimensional control variate for variance reduction. 
We remark that $ \kp(x,x') $ can be easily calculated with given $\k(x,x')$ and $\score_p(x)$, even when $p(x)$ is unnormalized. 

If we now replace the expectation $\E_p[\cdot]$ in \eqref{equ:Dhh} with $\E_q[\cdot]$ of a different smooth density $q(x)$ supported on $\X$,
\eqref{equ:Dhh} would not equal zero; instead, it gives a non-negative discrepancy measure of $p$ and $q$: 
\begin{align}
\label{equ:ksdexp}
\S(q, ~p) = \E_{x,x' \sim q} \big [  \kp(x,x') \big] \geq 0, 
\end{align}
where $\S(q,~p)$ is always nonnegative because $\kp(x,x')$ can be shown to be  positive definite if $\k(x,x')$ is positive definite \citep[e.g.,][Theorem 3.6]{liu2016kernelized}.
In addition, one can further show that $\S(q,~p)$ equals zero if and only if $p = q$ once $\k(x,x')$ is strictly positive definite in certain sense: 
strictly integrally positive definite 
in \citet{liu2016kernelized}, and $cc$-universal in  \citet{chwialkowski2016kernel}; 
these conditions are satisfied by common kernels such as the RBF kernel $\k(x, x') = \exp(-\frac{1}{2h^2} || x  - x'||^2_2)$, which is also easily in the Stein class of a smooth density $p(x)$ supported in $\X = \R^d$ because of its decaying property. 
\todo{We also note that $\S(p,q)$ can be viewed as the maximum mean discrepancy between $p$ and $q$ under kernel $\kp(x,x')$; the difference is, however, that $\kp(x,x')$ is $p$-specific, and hence $S(q,~ p)$ is asymmetric. }



One can further consider kernel $\kpp(x, x') = \kp(x,x') + 1$, whose corresponding RKHS $\Hpp$ consists of functions of form $h(x) + c$ with $h\in \Hp$ and $c$ is a constant in $\R$. 
%
Therefore, $\Hpp$ includes functions with arbitrary values of mean $\E_ph$.
Further, \citet{chwialkowski2016kernel} showed that $\Hpp$ is dense in $L^2(\X)$ when $\k(x, x')$ in $cc$-universal. 
As a consequence, 
$\Hp$ is dense in the subset of $L^2(\X)$ with zero-mean under $p(x)$, 
that is, for any $h \in L^2(\X)$ with $\E_p[h(x)]=0$ and any $\epsilon>0$, there exists $h'\in \Hp$ such that $|| h - h'||_{\infty} \leq \epsilon$. \redfix{[check]}

\section{Stein Importance Weights}\label{sec:stein}
Let $\{x_i\}_{i=1}^n$ be a set of points in $\R^d$ and we want to find a set of weights $\{w_i\}_{i=1}^n$, $w_i \in \R$, such that the weighted sample $\{x_i, w_i\}_{i=1}^n$ closely approximates the target distribution $p(x)$ in the sense that
$$
\sum_{i=1}^n w_i h(x_i) \approx \E_p [h(x)],
$$
for general test function $h(x)$. 
For this purpose, we define an empirical version of the KSD in \eqref{equ:ksdexp} to measure the discrepancy between $\{x_i, w_i\}_{i=1}^n$ and $p(x)$, 
$$
\S(\{x_i, w_i\}, ~ p) = \sum_{i,j=1}^n w_i w_j \kp(x_i, x_j) = \vw ^\top \KP \vw,
$$
where $\KP= \{ \kp(x_i, x_j) \}_{i,j=1}^n$ and $\vv w = \{ w_i \}_{i=1}^n$, and we assume the weights are self normalized, that is, $\sum_i w_i = 1$. 
We then select the optimal weights by minimizing the discrepancy $\S(\{x_i, w_i\}, ~ p)$,  
\begin{align}
\label{equ:vw}
\begin{split}
\hat {\vv w} =  \argmin_{\vv w}\bigg\{ \vv w^\top \KP \vv w,  ~~  s.t.~~ \sum_{i=1}^n w_i = 1, ~~~ w_i \geq 0\bigg\},
\end{split}
\end{align}
where in addition to the normalization constraint $\sum_i w_i = 1$, 
we also restrict the weights to be non-negative; these two simple constraints have important practical implications as we discuss in the sequel. 
Note that the optimization in \eqref{equ:vw} is a convex quadratic programming that can be efficiently solved by off-the-self optimization tools. For example, 
both mirror descent and Frank Wolfe take $\Od(n^2/\epsilon)$ to find the optimum with $\epsilon$-accuracy.\redfix{[reference]} 
Solving \eqref{equ:vw} does not require to know how the points $\{x_i\}_{i=1}^n$ are generated, and hence gives a \emph{black-box} importance sampling. 

Theoretically,  minimizing the empirical KSD can be justified by the following bound. 
\begin{pro}\label{pro:bound}
Let $h(x)$  be  a test function and $h - \E_ph  \in \Hp$. Assume $\sum_{i=1}^n w_i = 1$, 
we have 
\begin{align}\label{equ:bound}
|\sum_{i=1}^n w_i h(x_i)  ~- ~ \E_p h| ~ \leq ~ C_h \sqrt{\S(\{x_i, w_i\}, ~ p)}, 
\end{align}
where $C_h  =  || h  - \E_ph ||_{\Hp},$ which depends on $h$ and $p$, but not on $
\{x_i, w_i\}_{i=1}^n$. 
\end{pro}
%
\paragraph{Remark}
i) The condition $h - \E_ph  \in \Hp$ is a mild requirement
as we discussed in Section~\ref{sec:background}, 
because $\Hp$ is dense in the subset of $L^2(\X)$ with zero means under $p(x)$ when $\k(x,x')$ is $cc$-universal \citep{chwialkowski2016kernel}, 
for which many commonly used kernels satisfy. 

ii) \citet[Theorem 3]{oates2014control} has a similar result which does not require $\sum_i w_i  =1$, but 
has a constant term larger than $C_h$ when $\sum_i w_i = 1$ does hold\redfix{~(see Appendix)}. 
We propose to enforce $\sum_i w_i = 1$ because it gives exact estimation for constant functions $h(x)=c$, and is common practice for importance sampling (which is referred to as self-normalized importance sampling).  In our empirical results, we find that the normalized weights can significantly stabilize the algorithm, especially for high dimensional models. 


iii) One can show that the $\S(q, ~ p)$ as defined in \eqref{equ:ksdexp} can be treated as a maximum mean discrepancy (MMD) between $p$ and $q$, equipped with the ($p$-specific) kernel $\kp(x,x')$. In the light of this, bound \eqref{equ:bound} is a form of the worse case bounds of the kernel-based quadrature rules \citep[e.g.,][]{chen2010super, bach2015equivalence, huszar2012optimally, niederreiter2010quasi}. 
The use of the special kernel $\kp(x,x')$
allows us to calculate the discrepancy tractably for general unnormalized distributions; 
this is in contrast with the MMDs with typical kernels which are intractable to calculate 
due to the need for evaluating the a term related to the expectation of the kernel under distribution $p$. 

\todo{
\red{[]}
iii) 
Although this work mainly focuses on the weight optimization, 
one may consider a more general framework which also optimizes the point locations $\{x_i\}$ in addition to the weights $\{w_i\}$, e.g.,
$$
[\hat{\vv x}, ~ \hat {\vv w}] =  \argmin_{\vv w, \vv x} \bigg\{ \vv w^\top \KP \vv w,  ~~  s.t.~~ \sum_{i=1}^n w_i = 1, ~~~ w_i \geq 0\bigg\}. 
$$
Unfortunately, the optimization over $\hat{\vv x}$ is highly convex, and does not yield good empirical results (personal communication with Chris Oates and Lester Mackey). 
A better approach is discussed in a recent work \citep{liu2016stein}  
where it is shown it is better to optimize the KL divergence instead of the Stein discrepancy,
based on the observation that Stein discrepancy is linked to a type of gradient of KL divergence. 
}

\subsection{Practical Applications}\label{sec:practical}
Our method as summarized in Algorithm~\ref{alg:steinMCMC} can be used to refine any sample $\{x_i\}_{i=1}^n$ generated with arbitrary black-box mechanisms, 
and allows us to apply importance sampling in cases that are otherwise difficult. 
As an example, 
we can generate $\{x_i\}_{i=1}^n$ by simulating $n$ parallel MCMC chains for $m$ steps, 
where the length $m$ of the chains can be smaller than what is typically used in MCMC, because it just needs to be large enough to bring the distribution of $\{x_i\}_{i=1}^n$ ``roughly'' close to the target distribution. 
This also makes it easy to parallelize the algorithm compared with running a single long chain. 
In practice, one may heuristically decide if $m$ is large enough by checking the variance of the estimated weights $\{w_i\}_{i=1}^n$ (or the effective sample size). 
One can also simulate $\{x_i\}_{i=1}^n$ using MCMC with approximate translation kernels 
as these required for massive datasets \citep[e.g.,][]{welling2011bayesian, alquier2016noisy}, 
so our method provides a new solution for big data problems. 

We should remark that when $\{x_i\}_{i=1}^n$ is simulated from $n$ independent MCMC initialized from a distribution $q_0(x)$, the weight $w_0(x)=n^{-1} p(x)/q_0(x)$ does provide a valid importance sampling weights in that $\sum_i w_0(x_i) h(x_i)$ gives an unbiased estimator \citep[][Theorem 6.1]{maceachern1999sequential}. 
However, this weight does not update as we run more MCMC steps, and performs poorly in practice. 
\todo{
This basic fact is leveraged in sequentially Monte Carlo methods \citep[e.g.,][]{smith2013sequential, neal2001annealed}, which forms another way that combines the advantages of IS and MCMC; we leave the combination of our method with sequential Monte Carlo for future work. 
}

There are many other cases where black-box IS can find useful. 
For example, 
we can simulate $\{x_i\}_{i=1}^n$ from bootstrapping or perturbed maximum \emph{a posteriori} (MAP) \citep{papandreou2011perturb, hazan2013sampling}, 
that is, $x_i = \argmax_x \tilde p( x)$ where $\tilde p(x)$ is a perturbed version of $p(x)$, or the bootstrapping likelihood. 
The idea of using importance weighted bootstrapping to carry out Bayesian calculation has been discussed before \citep[e.g.,][]{efron2012bayesian}, but was limited to simple cases when the bootstrap distribution is computable. 
Black-box IS can also be used to refine the results of
variational inference \citep[e.g.,][]{wainwright2008graphical}, 
especially for 
the cases with complex variational 
proposal distributions \citep[e.g.,][]{rezende2015variational}.  
%
\todo{
Black-box IS can also be used to refine the results of
variational inference \citep{wainwright2008graphical}, 
which has been widely used in machine learning, and often lack consistency guarantees. 
We may simulate $\{x_i\}$ from the variational proposal distribution estimated by minimizing the KL divergence, 
turning it into a theoretically consistent estimator. 
Importantly, black-box IS does not requires to calculate the proposal likelihood, and can be used for complex proposal distributions such as \citep{rezende2015variational}.  
}

\begin{algorithm}[tb] %
\caption{Stein Importance Sampling} \label{alg:steinMCMC}
\begin{algorithmic}
\STATE 1. Generate $\{x_i\}_{i=1}^n$ using any mechanism that is believed to resemble $p(x)$ (e.g., by running $n$ independent MCMC chains for a small number of steps, or using parametric bootstrap).
\STATE 2. Calculate importance weights for $\{\hat w_i\}_{i=1}^n$ by \eqref{equ:vw}. 
\STATE 3. Calculate $\sum_i \hat w_i h(x_i)$ to approximate $\E_p[ h] $for test function $h$. 
\end{algorithmic}
\end{algorithm} 



\subsection{Convergence Rate}\label{sec:convergence}
Our procedure does not assume the generation mechanism of $\{x_i\}_{i=1}^n$, 
but if $\{x_i\}_{i=1}^n$ is indeed generated ``nicely'', 
error bounds can be easily established using  Proposition~\ref{pro:bound}: 
if there exists a set of ``reference'' positive normalized weights $\{w^*_i\}_{i=1}^n$ such that $\S(\{x_i, ~ w_i^*\}, ~ p) = \Op(n^{-\delta})$, then 
the mean square error of our estimator with weight $\{\hat w_i\}_{i=1}^n$ returned by \eqref{equ:vw} should also be $\Op(n^{-\delta})$ by 
following \eqref{equ:vw} and \eqref{equ:bound}.  

To gain more intuition, assume $\kp(x,x')$ has a set of eigenfunctions $\{\phi_\ell(x)\}$ and eigenvalues $\{\lambda_\ell\}$ such that $\kp(x,x') = \sum_\ell \lambda_\ell \phi_\ell(x) \phi_\ell(x')$, then we have \text{\jasonQaddressed{What is the outer expectation on the LHS over? The $x_i$'s?}} 
\begin{align*}
|\sum_i \hat w_i h(x_i) - \E_p h |^2
&  \leq  C_h^2 ~  {  \S(\{ x_i, \hat w_i\}, ~ p) }  \\
 & \hspace{-.05\textwidth} =  C_h^2 ~ \sum_{\ell} \lambda_\ell ~ (\sum_i w_i\phi_\ell(x_i) - \E_p \phi_\ell)^2, 
\end{align*}
where $C_h = || h - \E_p h ||_{\Hp}$ and we used the fact that $\E_p\phi_\ell = 0$ since $\phi_\ell \in \Hp$. 
Therefore, it is enough to find a set of positive and normalized reference weights 
whose error on estimating $\E_p\phi_\ell$ is low. 
Note that such reference weight does not necessarily need to be practically computable to establish the bound. 

As an obvious example, when $\{x_i\}_{i=1}^n$ is i.i.d. drawn from an (unknown) proposal distribution $q(x)$, the typical importance sampling weight $w^*_i \propto  p(x_i) / q(x_i)$ (up to the normalization)
can be used as a reference weight 
to establish an $\Op(n{}^{-1/2})$ error rate as the typical Monte Carlo methods have.
%
%
{
\begin{thm}\label{thm:IS}
Assume $h-\E_p h \in \Hp$ and $\{x_i\}_{i=1}^n$ is i.i.d. drawn from $q(x)$ with the same support as $p(x)$. Define $w_*(x)=p(x)/q(x)$ and assume
$\E_{x\sim p} (| w_*(x) \kp(x,x)|) < \infty$, and 
$\E_{\{x,x'\}\sim p} [ w_*(x) w_*(x') \kp(x,x')^2 ] < \infty$. 
For $\{\hat w_i\}_{i=1}^n$ defined in \eqref{equ:vw}, we have
$$ |\sum_{i=1}^n \hat w_i h(x_i) - \E_p h |  = \Op({n^{-1/2}}). $$
\end{thm}
}
Interestingly, it turns out the typical importance weight $w_*(x)\propto p(x)/q(x)$ is not the best possible reference weight; 
better options can be constructed using various variance reduction techniques to give convergence rates better than the typical $\Op(n^{-1/2})$ rate.  
%
\begin{thm}\label{thm:mainb0}
Assume $\{x_i\}_{i=1}^n$ is i.i.d. drawn from $q(x)$ and $w_*(x)=p(x)/q(x)$. 
Let $\{\phi_\ell\}_{\ell=1}^\infty$ be the set of orthogonal eigenfunctions w.r.t. $p(x)$ with eigenvalues $\{\lambda_\ell\}_{\ell=1}^\infty$. 
Assume all the following quantities are upper bounded by $M$ uniformly for $\forall x\in \X$\emph{:} $\sum_{\ell=1}^\infty \lambda_\ell$, $w_*(x)$, $|\phi(x)|$, $\max_{\ell,\ell'}\var_{x\sim q}[w_*(x)^2 \phi_\ell(x)\phi_\ell'(x)]$, we have
$$\E_{\vx \sim q} \big [   |\sum_i \hat w_i h(x_i) - \E_p h |^2 \big ]^{1/2} = \bigO{n^{-(1+\alpha)/2}}, $$
where $\alpha$ is a number that satisfies $0< \alpha\leq 1$ and is decided by the bound $M$ and the decay of the eigenvalues $\res(n) = \sum_{\ell > n} \lambda_\ell$ of kernel $k_p(x,x')$.  See Theorem~B.5 in Appendix for more details. 
\end{thm}
The proof of Theorem~\ref{thm:mainb0} (see Section~2.2 in Appendix) is based on first constructing a set of (possibly negative) reference weights using a control variates method based on the orthogonal basis functions $\{\phi_\ell\}$, and then zero out the negative elements and normalize the sum to obtain a set of positive normalized reference weights. 
This is made possible because the initial reference obtained by control variates 
is mostly likely positive, since they can be treated as a perturbed version of the typical (positive) importance weights $w_*(x)\propto p(x)/q(x)$ as used in 
Theorem~\eqref{thm:IS}, where the perturbation is introduced 
to cancel the estimation error and increase the accuracy.  
Therefore, zeroing out the negative elements does not have significant impact on the error bound compared with the initial reference weights. 
This provides a justification on the non-negative constraint, and allows us to construct a set of non-negative reference weights for our proof.  
The non-negative constraint, also known as the \emph{garrote} constraint \citep{breiman1995better}, is also motivated by
the empirical observation that it gives more stable results for small sample size $n$ (intuitively, it seems hard to believe that a large negative weight
would give improvements when $n$ is small, unless the points $\{x_i\}_{i=1}^n$ were introduced in a careful way). 


The proof of Theorem~\ref{thm:mainb0} (see Section~2.2 in Appendix) is based on constructing a reference weight using a control variates method based on the orthogonal basis functions $\{\phi_\ell\}$. Our constructed reference weights can be treated as  a perturbed version of the typical importance weights $w_*(x)\propto p(x)/q(x)$ as used in 
Theorem~\eqref{thm:IS}, where the perturbation is introduced to cancel the estimation error and increase the accuracy.  Since this reference weights concentrate around $w_*\propto p(x)/q(x)$ which is positive, we can zero out its negative values without much impact on the error bound. This provides a justification on the non-negative garrote constraint, and allows us to construct a set of non-negative reference weights for our proof.  
\todo{
Because the bound in \eqref{equ:bound} works for both positive and negative weights, enforcing the non-negative constraint can only deteriorate the error bound established based on \eqref{equ:bound}.  
Nevertheless, the non-negative restriction is motivated by empirical study that shows significant improvement for small $n$ (when things are noisy, using negative weights would likely deteriorate the performance); theoretically, this is justified by the intuition that the optimal reference weights is likely to be close to the typical importance weight $p(x)/q(x)$ (but adding perturbations in a subtle way) and hence satisfy $w_i \geq 0$ with high probability. The result in Theorem~B.6 in Appendix shows that the non-negative constraint has no influence on the convergence rate if $\res(n)$ decays faster than $O(n^{-1})$. 
 }

Similar theoretical analysis can be found in \citet{briol2015probabilistic, bach2015equivalence}. 
In particular, \citet{briol2015probabilistic} used a similar ``reference weight'' idea to establish a convergence rate for Bayesian Monte Carlo. 
The main technical challenge in our proof is to 
make sure that the reference weight satisfies the non-negative and self-normalization constraints. 
Section~2.2 in Appendix provides more detailed discussions. 
\subsection{Other ``Super-Efficient'' Weights}
\label{sec:other}
We review several other types of  ``supper-efficient'' weights that also give better convergence rates than the typical $\Op(n^{-1/2})$ rate; 
this includes Bayesian Monte Carlo and the related (linear) control variates method, 
as well as methods based on density approximation of the proposal distributions, 
which can be interpreted as \emph{multiplicative} control variates \citep{nelson1987control} that reduce the variance. 

\paragraph{Bayesian Monte Carlo and Control Variates} 
Bayesian Monte Carlo \citep{o1991bayes, ghahramani2002bayesian} 
was originally developed to 
evaluate integrals using Bayesian inference procedure with Gaussian prior, which turns out to be equivalent to a weighted form $\sum_i w_i h(x_i)$ with $w_i$ being a set of weights independent of the test function $h$;
unlike our method, these weights are not normalized to sum to one and can take negative values.  

From a RKHS perspective, 
one can interpret Bayes MC as approximating $\E_p[h(x)]$ 
with $\E_p[\hat h(x)]$ where $\hat h(x)$ is an approximation of $h(x)$ constructed by kernel linear regression 
based on the data-value pair $\{x_i, h(x_i)\}_{i=1}^n$. 
Let $\k_0(x,x')$ be the kernel used in Bayes MC,
then one can show that Bayes MC estimate equals $\sum_i \hat w_i h(x_i)$
with 
$\hat {\vv{w}} =  [\hat w_i]_{i=1}^n =  (\vv K_0 + \lambda I)^{-1}\vv b$, where 
$\vv K_0 = [\k_0(x_i,x_j)]_{ij}$ and 
$\vv b = [\E_{x\sim p}(\k_0(x,x_i))]_{i=1}^n$, and $\lambda$ a regularization coefficient. 
Equivalently, Bayes MC can be treated as minimizing the maximum mean discrepancy (MMD)
between $\{x_i, w_i\}$ and $p$, with a form of
$$
\hat{\vv w}  = \arg\min_{\vv{w}} \big \{ \vv w ^\top \vv K_0 \vv w - 2 \vv b^\top \vv w + \lambda || \vv w||_2^2\big\}. 
$$
One of the main difficulty of Bayesian MC, however, 
is that it depends on $\vv b = \E_{p}[k_0(x,x')]$, 
which can be intractable to calculate for complex $p(x)$. 

The control variates method \citep[e.g.,][]{liu2008monte} also relies on a (kernel) linear regressor $\hat h(x)$, but 
estimates $\E_ph$ with a bias-correction term $\frac{1}{n}\sum_{i=1}^n (h(x_i) - \hat h(x_i)) + \E_p [\hat h(x_i)]$, which can also be rewritten into a weighted form. 
Note that when $\lambda = 0$ and $\vv K_0$ is strictly positive definite, 
the $\hat h(x_i)$ becomes an interpolation of $h(x)$ (i.e., $h(x_i)=\hat h(x_i)$), and control variates and Bayes MC becomes equivalent. 
In control variates, one can also use only a subset of the data to estimate $\hat h(x)$ and use the remaining data to estimate 
the expectation of the difference $h(x) - \hat h(x)$; 
 this ensures the resulting estimator is unbiased. 
 
 Theoretically, the convergence rate of control variates and Bayesian Monte Carlo 
can both be established to be $\Op(n^{-(1+\alpha)/2})$, where $\alpha$ depends on how well $\hat h(x)$ can approximate $h(x)$; see \citet{oates2014control, oates2016convergence, briol2015probabilistic, briol2015frank, bach2015equivalence} for detailed analysis. 


Closely related to our work,
 \citet{oates2014control} and \citet{briol2015probabilistic} proposed to use the Steinalized kernel $\kpp(x,x') = \kp(x,x')+1$ 
in  control variates and Bayesian MC, respectively,\footnote{$\kp(x,x')$ can be not used directly in Bayesian Monte Carlo since it only includes functions with zero mean.} 
for which $\vv b  = \E_{x\sim p}[ \kpp(x,x')] = 1$. 
%
We can show that their method is equivalent to using the following weight 
\begin{align*}
\hat {\vv w} = \argmin_{\vv w} \{ \vv w^\top \KP \vv w + (\sum_i w_i - 1)^2 + \lambda ||w ||_2^2 \}. 
\end{align*}
This form is similar to our \eqref{equ:vw}, 
but does not enforce the non-negative garrote
constraint \citep{breiman1995better}
and replacing the normalization constraint $\sum_i w_i = 1$ with a quadratic regularization with regularization coefficient of one.
Here the L2 penalty $\lambda ||w ||_2^2$ is necessary for ensuring numerical stability in practice. 
In our case, it is the non-negative constraint that helps stabilize the optimization problem, without needing to specify a regularization parameter. 

\todo{
 \citet{oates2014control} referred his method a \emph{control functional} method because $\kpp(x,x')$ corresponds to an infinite dimensional feature map. 
The convergence rate of control variates and Bayesian Monte Carlo 
are both established to be $\Op(n^{-(1+\alpha)/2})$, where $\alpha$ depends on how well $\hat h(x)$ can approximate $h(x)$; see \citet{oates2014control, oates2016convergence, briol2015probabilistic, briol2015frank, bach2015equivalence} for detailed analysis. 
}

\vspace{-.1\baselineskip}
\paragraph{Approximating the Proposal Distribution}
Another (perhaps less well known) set of methods are based on replacing the importance weight $w_*(x) = p(x) / q(x)$ with an approximate version $\tilde w(x) = p(x) / \hat q(x)$, where $\hat q(x)$ is an estimator of proposal density $q(x)$ from $\{x_i\}_{i=1}^n$. 
While we may naturally expect that such approximation would decrease the accuracy compared with the typical IS that uses the exact $q(x)$, 
surprising results \citep{henmi2007importance, delyon2014integral} show that in certain cases the approximate weights $\tilde w(x)$ actually improve the accuracy. 
To gain an intuition why this can be the case, observe that we have $\tilde w(x) = [p(x)/q(x)] \cdot [q(x) / \hat q(x)]$, 
where the second term $q(x)/\hat q(x)$ may acts as a (multiplicative) control variate \citep{nelson1987control} which can decrease the variance 
if it is negatively correlated with the rest parts of the estimator. 
For asymptotic analysis, it is common to expend multiplicative control variates using Taylor expansion, 
which reduces it to linear control variates.

In particular, \citet{henmi2007importance} showed that when $q(x)$ is embedded in a parametric family $\mathcal Q = \{q(x \cd \theta)$, $\theta \in \Theta\}$, 
replacing $w_*(x)$ with 
the approximate weight $\tilde w(x) = p(x)/ \hat q(x)$, where $\hat q$ is the maximum likelihood estimator of $q(x)$ within $\mathcal Q$,
 would guarantee to decrease the asymptotic variance compared with the standard IS. 
The result in \citet{delyon2014integral} forms a 
non-parametric counterpart of \citet{henmi2007importance}, 
in which it is shown that taking $\hat q(x)$ to be a leave-one-out kernel density estimator of $q(x)$ would give super-efficient error rate $\Op(n^{-(1+ \alpha)/2})$ where $\alpha$ is a positive number that depends on the smoothness of $q(x)$ and $p(x)h(x)$. 
\todo{Both the weights defined in \citet{delyon2014integral} and \citet{oates2014control, oates2016convergence} are (typically) positive, but are not self-normalized. }
%
\todo{
\textbf{Further Discussion}
We should point out that the weights choose by Bayesian MC can also be interpreted as minimizing a MMD-based discrepancy, except without the normalization and positivity constraint. 
For example, consider the Bayesian Monte Carlo using kernel $\kpp(x,x') = \kp(x,x') + 1$ ($\kp(x,x')$ can be not used directly since it only includes functions with zero mean), we can show that its weights can 
$$
\hat {\vv w} = \argmin_{\vv w} \{ \vv w^\top \KP \vv w + (\sum_i w_i - 1)^2 \}. 
$$
This form is similar to our \eqref{equ:vw}, but does not strictly enforce the positivity and replacing the normalization constraint $\sum_i w_i = 1$ with a corresponding quadratic regularization.  
One may further apply a regularization coefficient on the quadratic regularization term, but it introduces an addition tuning parameter. 
In our experiments, we find that enforcing the normalization $\sum_i w_i =1$ can significantly stabilize the algorithm performance, especially in high dimensions. This \emph{self-normalization} idea has been well discussed in importance sampling literature \citep[e.g.,][]{}, but seems to be less discussed from the Bayesian Monte Carlo perspective. 
}
%
\todo{
\textbf{Related Works}
On the other hand, Quasi Monte Carlo type methods deterministically select $\{\vv x, \vv w\}$ by minimizing certain discrepancy $\S(\{\vv x, \vv w\}; ~ p)$ to $p(x)$ which are upper bounds of the approximation error in the sense of 
$$
|\sum_i w_i h(x_i)  ~- ~ \E_p(h(x))| ~ \leq ~ C_h \cdot \sqrt{\S(\{x_i, w_i\}, ~ p)}, 
$$
where $C_h$ is a constant related to $h$. 
The discrepancy is often taken to be a maximum mean discrepancy \citep{smola2007hilbert}; 
We remark that standard Quasi Monte Carlo \citep[e.g.,][]{niederreiter2010quasi} and \citep[][]{chen2010super, bach2015equivalence} usually assumes uniform weights, while Bayesian quadrature further optimizes the weights  \citep[e.g.,][]{o1991bayes, ghahramani2002bayesian, briol2015probabilistic}. These methods, however, only work with relatively simple $p(x)$. 
}

\begin{figure*}[t]
   \centering
\scalebox{.9}{   
   \begin{tabular}{ccccc}
 \hspace{-1.5em}   \raisebox{1.5em}{  \includegraphics[height=0.14\textwidth]{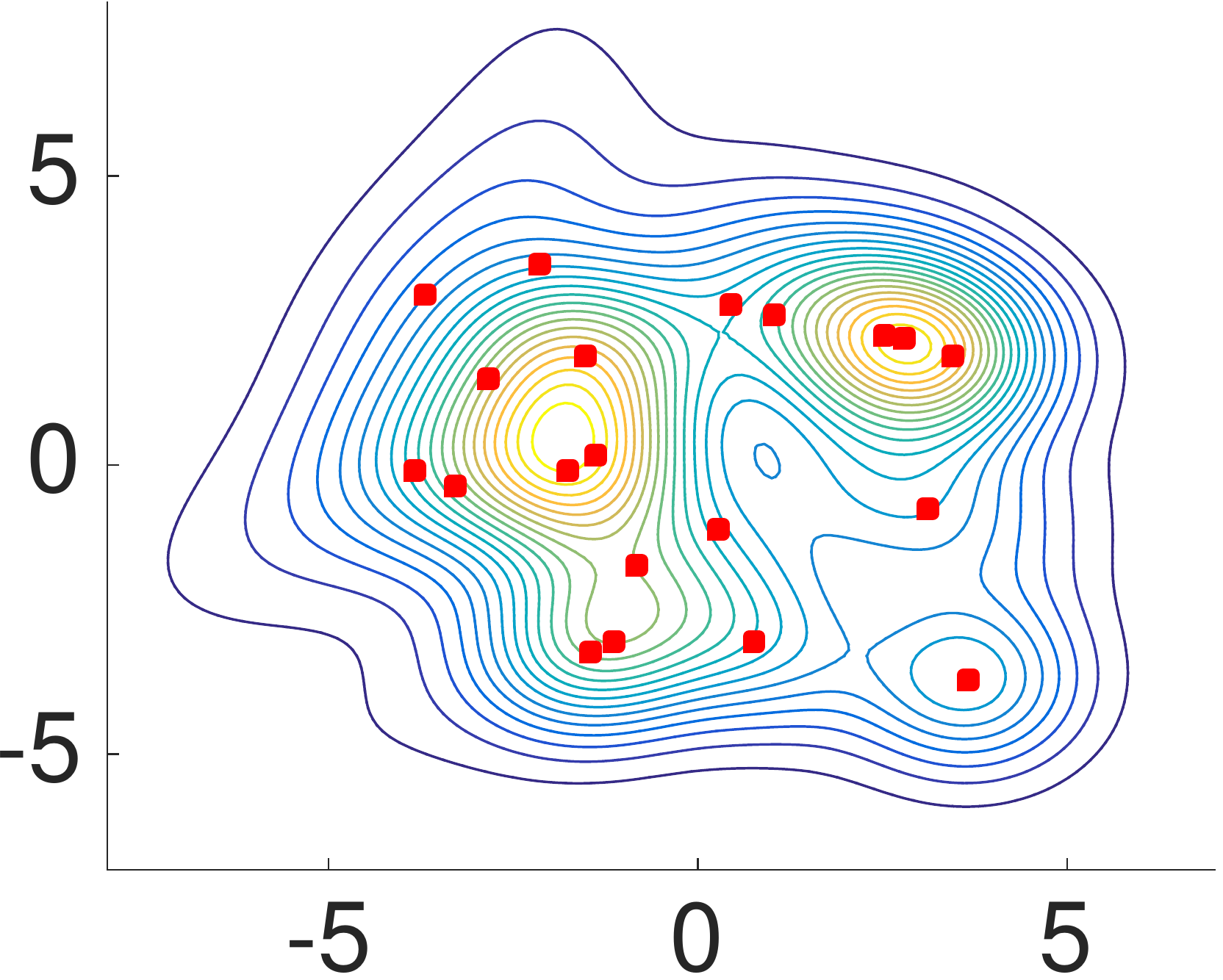} } \hspace{-2em}&
     \hspace{-1.8em}   \includegraphics[height=0.2\textwidth]{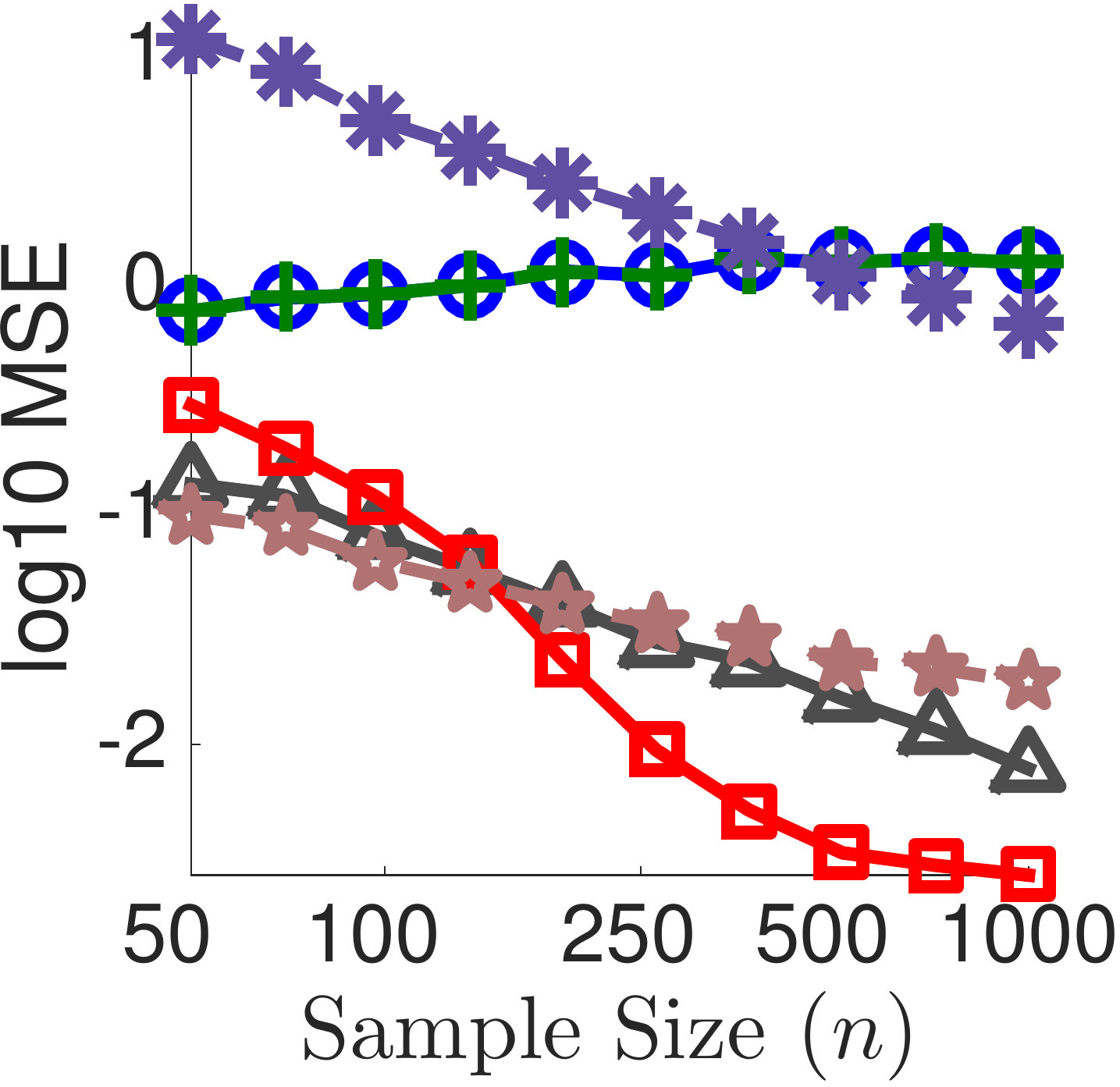} &
 \hspace{-1em}  \raisebox{0em}{  \includegraphics[height=0.2\textwidth, trim={1.4cm 0 0 0},clip]{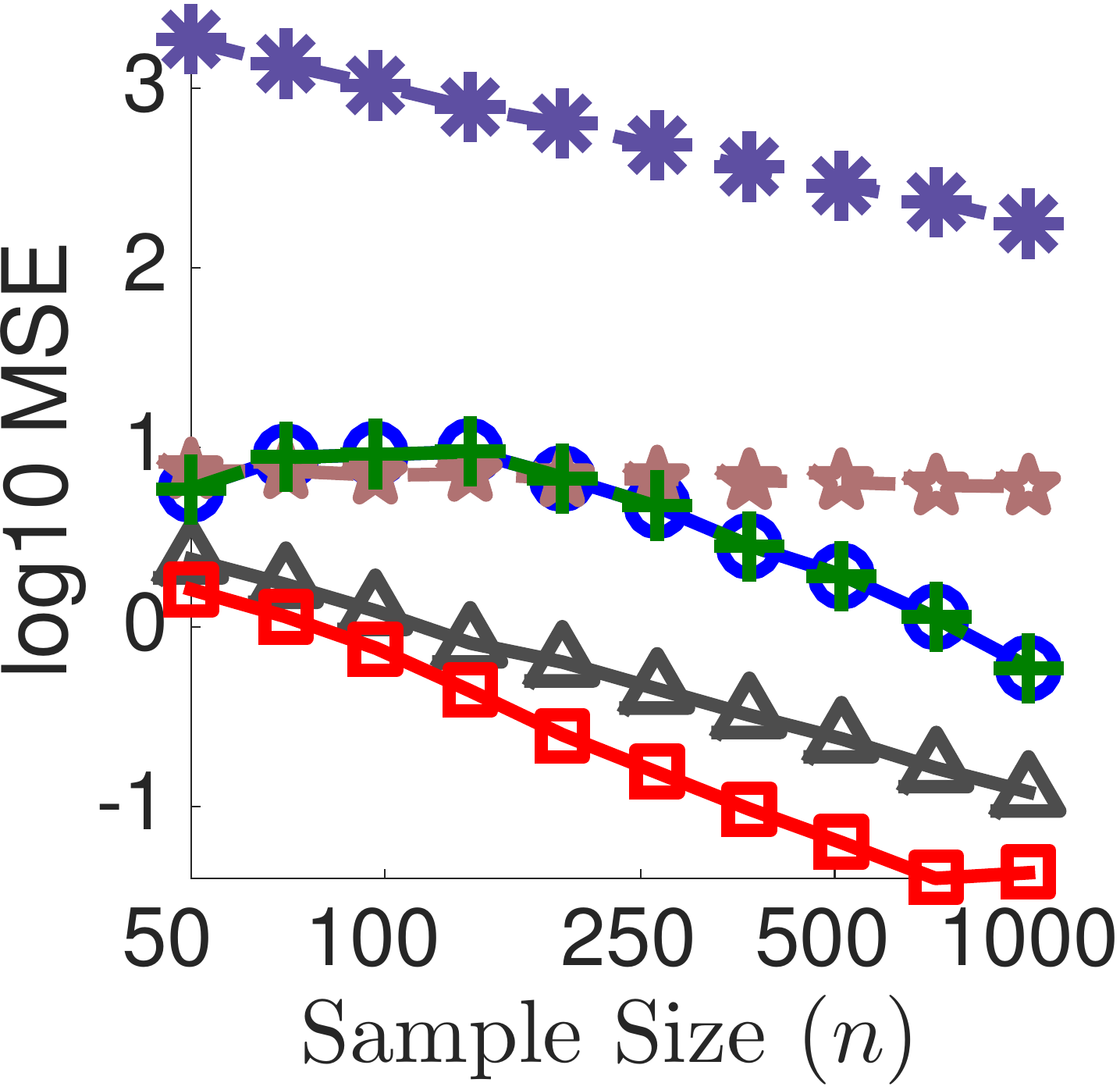}  }&
 \hspace{-1em}  \raisebox{0em}{  \includegraphics[height=0.2\textwidth, trim={1.4cm 0 0 0},clip]{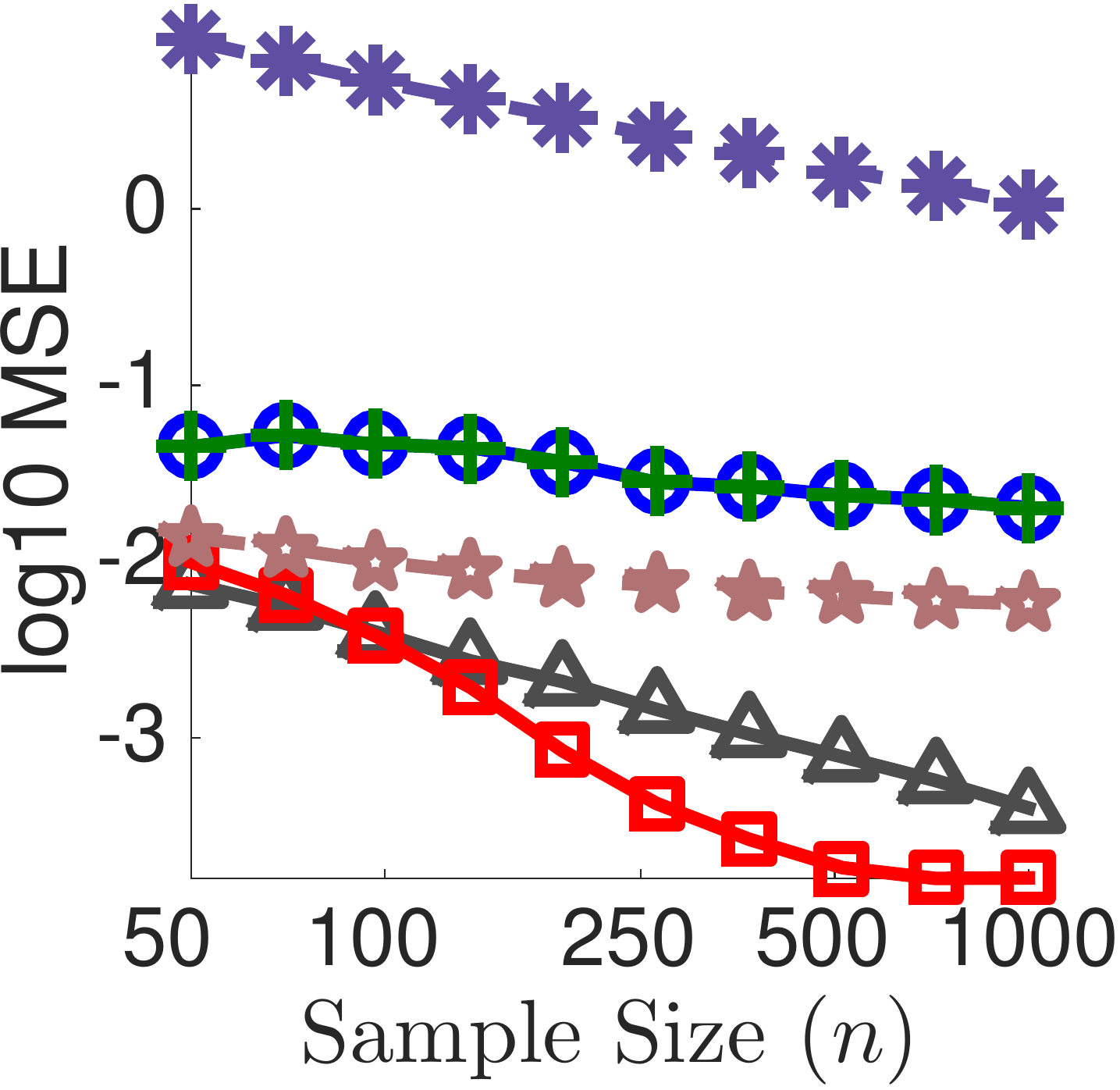}  }&
 \hspace{-2em}  \raisebox{4em}{  \includegraphics[height=0.1\textwidth]{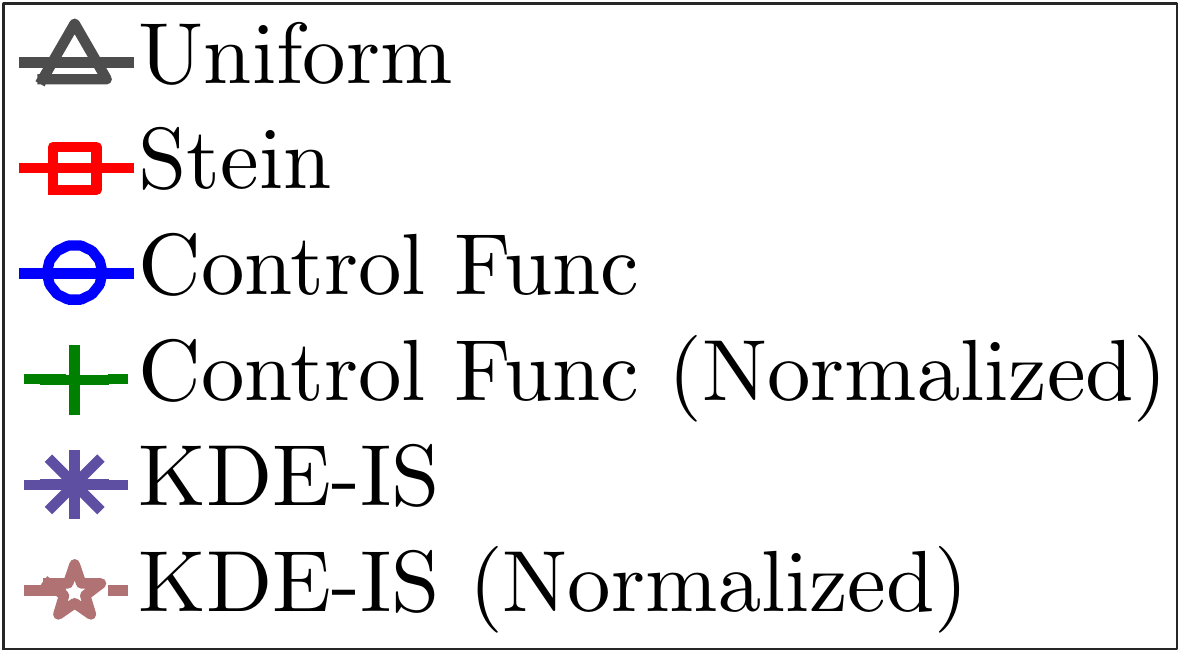}  } \\  
{ \small (a) Contour of $p(x)$} & {\small(b) $\E(x)$} & {\small(c) $\E(x^2)$} & {\small(d) $\E(\cos(\omega x+b))$} &
   \end{tabular}
   }
   \caption{\capfontsize Gaussian Mixture Example. (a) The contour of the distribution $p(x)$ that we use;
   the red dots represent the centers of the mixture components. 
The sample $\{x_i\}$ is i.i.d. drawn from $p(x)$ itself. 
   (b) - (c) The MSE of the different weighting schemes
   for estimating $\E_p h$, when $h(x)$ equals $x$, $x^2$, and $\cos(\omega x+b)$, respectively. For $h=\cos(\omega x+b)$ in (d), we draw $\omega\sim\normal(0,1)$ and $b \sim \textrm{Uniform}([0,2\pi])$ and average the MSE over $20$ trials.}
   \label{fig:gmm}   
\end{figure*}

\begin{figure*}[t]
   \centering
\scalebox{.9}{      
   \begin{tabular}{cccc}
      \includegraphics[width=0.25\textwidth]{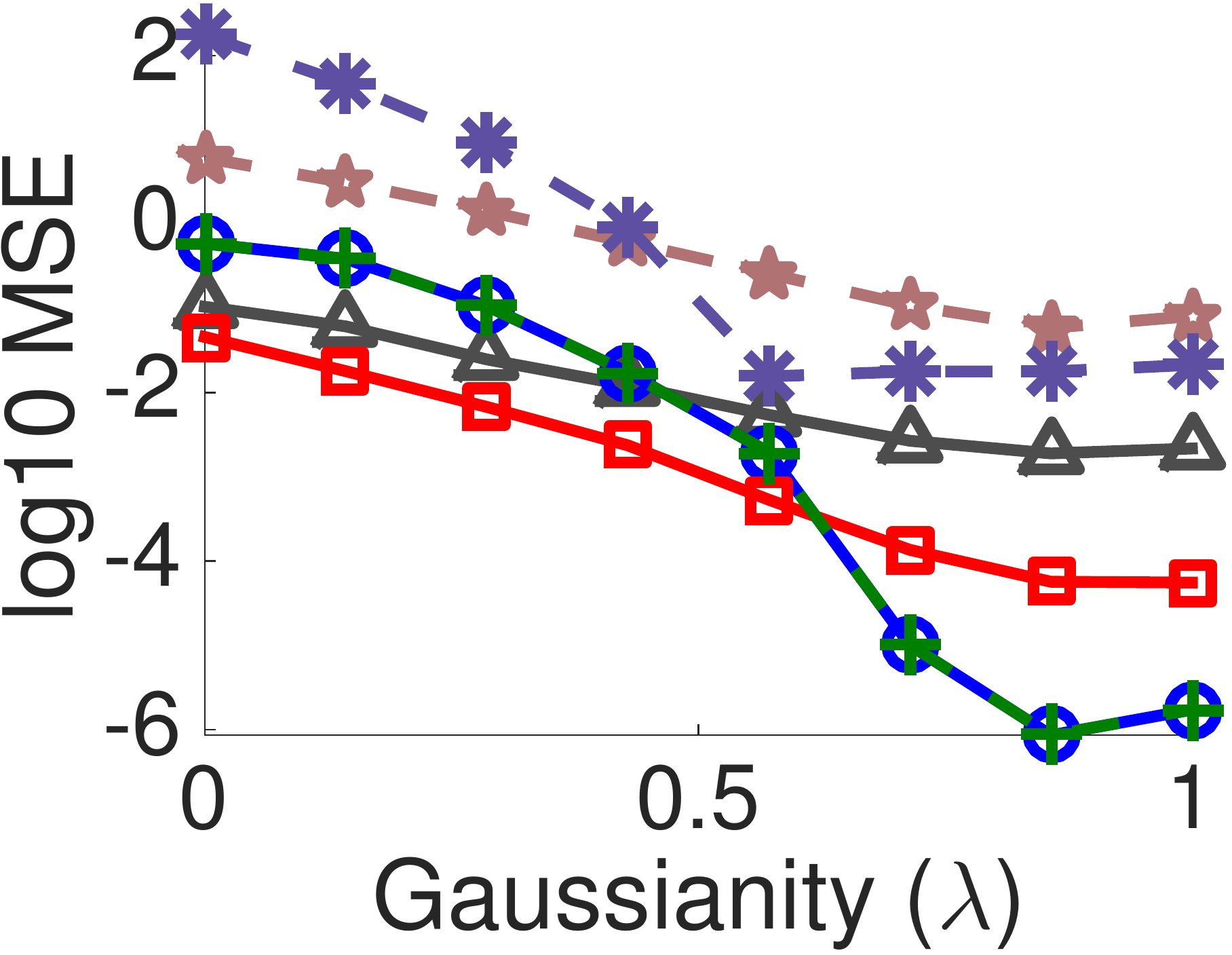}
      &\hspace{.1\textwidth}
      &   
   \includegraphics[width=0.25\textwidth]{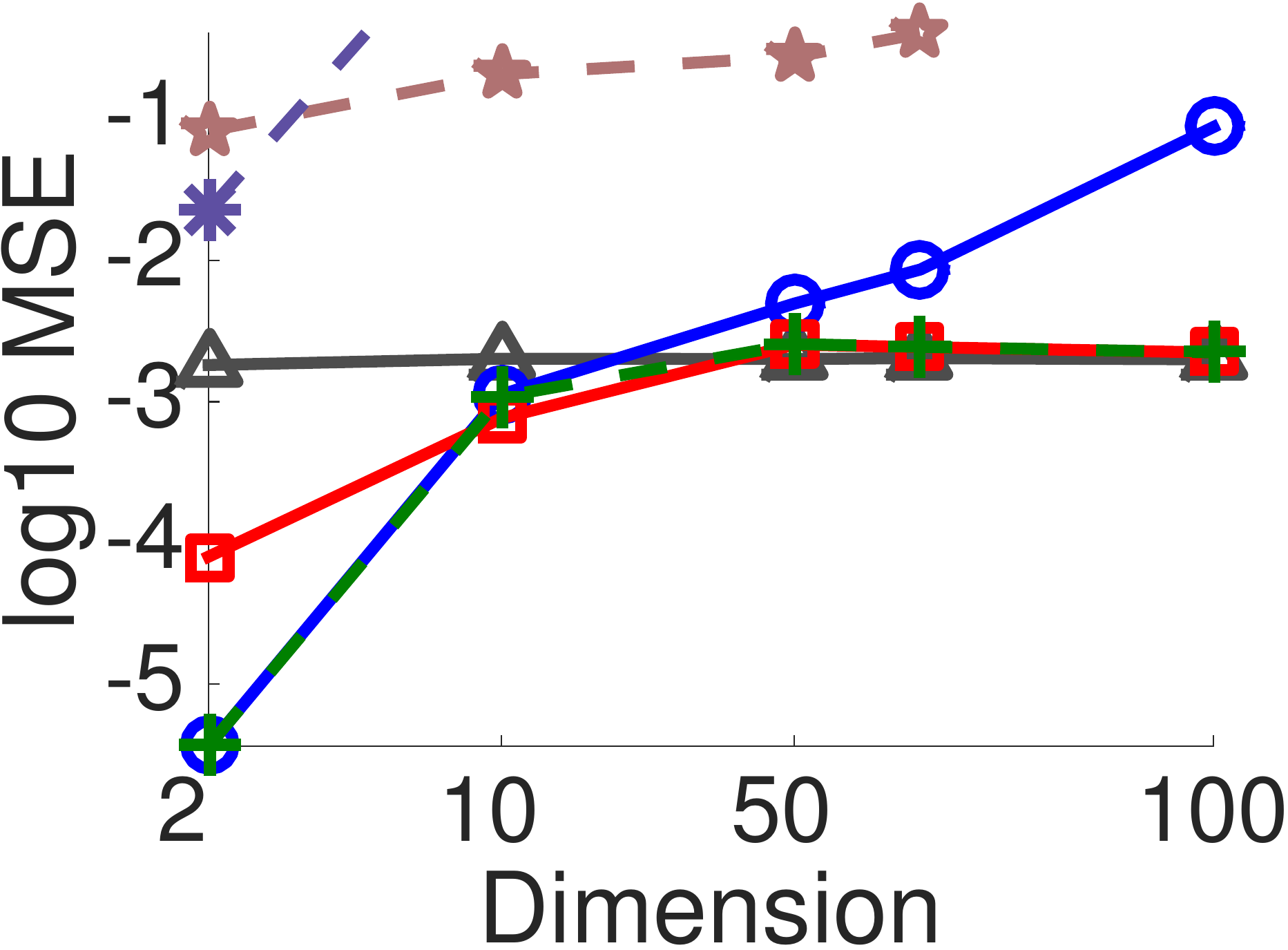} &
\hspace{0em} \raisebox{4em}{   \includegraphics[width=0.2\textwidth]{figures/Leg_51steinISmala0d0039486_binprobitn1000d50wtrue1_50_0mu0invS0d01_NVec50_10_1000_sVec100.pdf}  } \\
 { \small(a) Varying Gaussianity} && {\small(b) Varying Dimension}  &
   \end{tabular}
   }
   \caption{\capfontsize
   (a) Results on $p(x | \lambda)$ where $\lambda$ indexes the Gaussianity: $p(x|\lambda)$ equals $\normal(0,1)$ when $\lambda=1$ and it reduces to the $p(x)$ in Figure~\ref{fig:gmm}(a) when $\lambda=0$. (b) Results on standard Gaussian distribution with increasing dimensions. The sample size is fixed to be $n=100$ in both (a) and (b). 
   \todo{We see that {\tt Control Func} performs the best for low dimensional Gaussian, but seems to degenerate for non-Gaussian and high dimensional cases;  {\tt Control Func (Noramlized)} stablizes the result in high dimensions as shown in (b).}
    The MSE is for estimating $\E(x^2)$.}
   \label{fig:gmm2}
\end{figure*}

\begin{figure}[t]
   \centering
\scalebox{.9}{      
   \begin{tabular}{c}
       \includegraphics[width=0.25\textwidth]{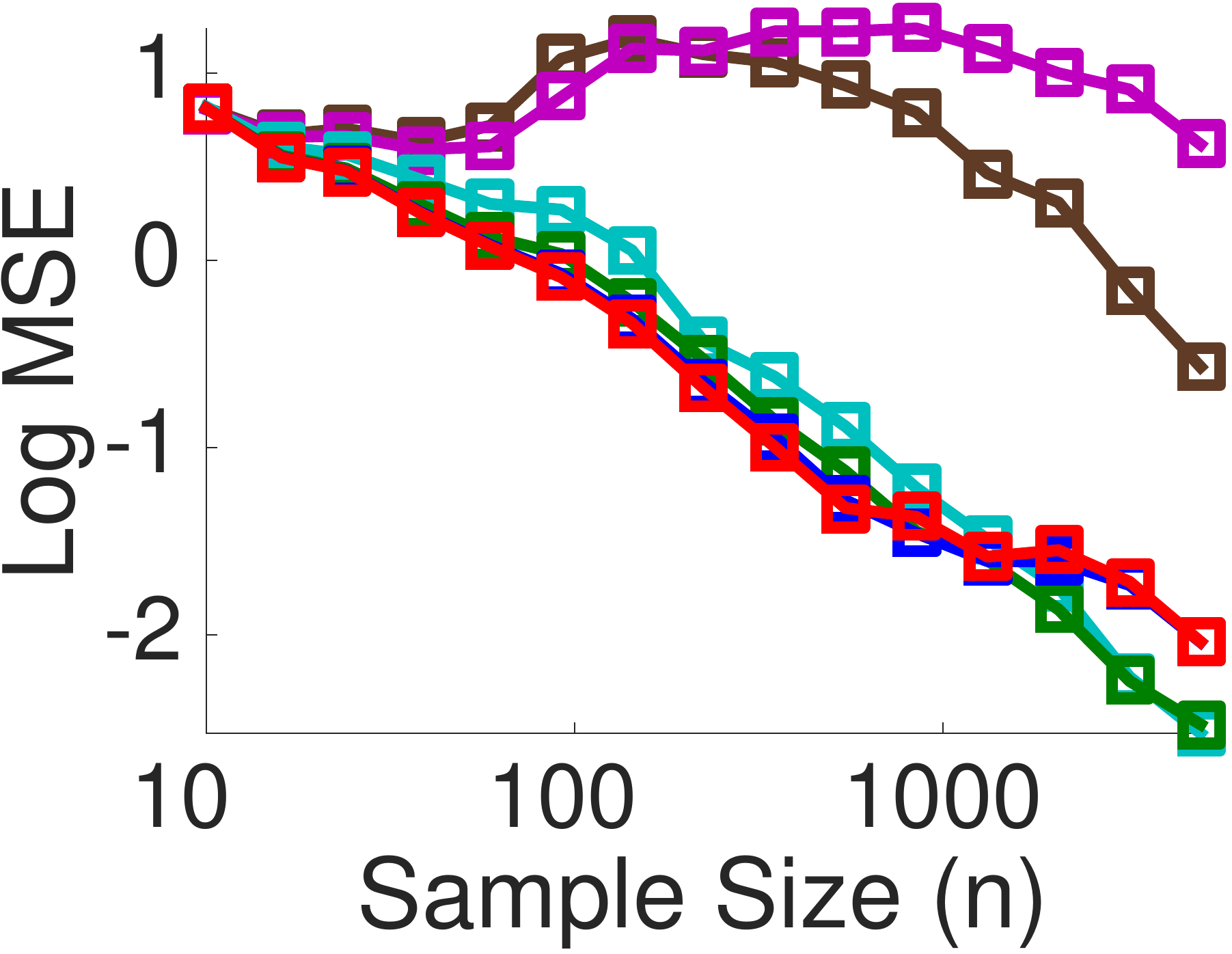}
       \raisebox{1em}{   \includegraphics[width=0.1\textwidth]{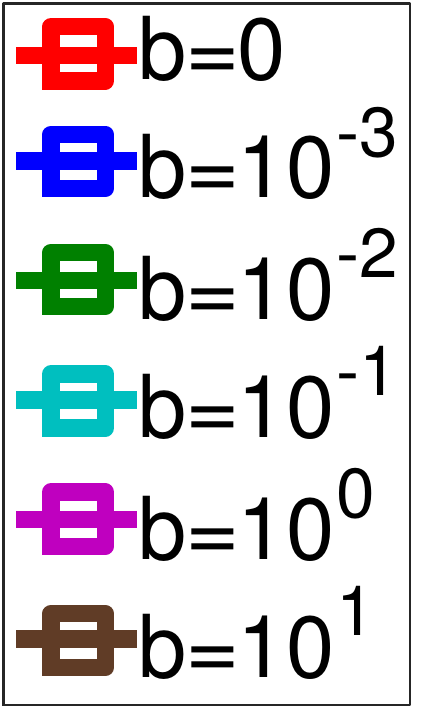}} \\    
 {\small (c) Stein with relaxed $w_i\geq -b$ constraint}
   \end{tabular}
   }
   \caption{\capfontsize
   The result of our method on the $p(x)$ in Figure~\ref{fig:gmm}(a) when the non-negative constraint $w_i \geq 0$ replaced by a general lower bound $w_i \geq -b $ with different values of $b$. 
    The MSE is for estimating $\E(x^2)$.}
   \label{fig:gmm2lowerbound}
\end{figure}

\section{Experiments}\label{sec:experiments}
We  empirically evaluate our method and compare it with the methods mentioned above, first on an illustrative toy example based on Gaussian mixture, and then on Bayesian probit regression. 
The methods we tested all have a form of $\sum_i w_i h(x_i)$, where the weights $w_i$ are decided by one of the following algorithms: 

1. Uniform weights $w_i = 1/n$ ({\tt Uniform}). 

2. Our method that solves \eqref{equ:vw} (referred as {\tt Stein}), for which we use RBF kernel 
$k(x,x') = \exp(-\frac{1}{h}||x-x'||^2_2)$; the bandwidth $h$ is heuristically chosen to  be the median of the pairwise square distance of data $\{x_i\}_{i=1}^n$
as suggested by \citet{gretton2012kernel}.  


3. The control functional method {\tt Control Func} following the empirical guidance in \citet{oates2014control}, which is also equivalent to Bayesian MC with kernel $\kpp(x,x') = \kp(x,x')+1$.
Note that the weights $\{ w_i\}$ in this method may be negative and do not necessarily sum to one. We also test a modified version of it $\sum_i w_i h(x_i) /\sum_i w_i$
that normalizes the weights and refer it as {\tt Control Func (Normalized)}. The kernel $k(x,x')$ and the bandwidth are taken to be the same as our method. 
We follow \citet{oates2014control}'s guidance to select that an L2 regularization coefficient to stabilize the algorithm. 

4. The kernel density estimator (KDE) based method by \citet{delyon2014integral} ({\tt KDE}), which uses weights $w_i =n^{-1} p(x_i)/\hat q_i(x_i)$, where $\hat q_i(x)$ is a leave-one-out KDE of form $\hat q_i(x) = \sum_{j\neq i} k(x, x_j)/n$. We report the result when using RBF kernel with bandwidth decided by the rule of thumb 
$h = \hat \sigma \big(\frac{d2^{d+5} \Gamma(d/2+3)}{(2d+1)n}\big)^{1/(4+d)}$, where $\hat \sigma$ is the standard deviation of $\{x_i\}_{i=1}^n$ and $d$ is the dimension of $x$. 
We also tested the choice of kernel and bandwidth suggested in \citet{delyon2014integral} but did not find consistent improvement. 
Similar to the case of the control functional method, we also test a self-normalized version of {\tt KDE} and denote it by {\tt KDE (Normalized)}. 

We evaluate these methods by comparing their mean square errors (MSE) for estimating $\E_p h$, with $h(x)$ taken to be $x$, $x^2$ or $\cos(\omega x + b)$. 
For $h(x) =\cos(\omega x+b)$, we draw $\omega\sim\normal(0,1)$ and $b \sim \textrm{Uniform}([0,2\pi])$ and average the MSE over $20$ random trials. 

\paragraph{Gaussian Mixture}
We start with a $2$-D Gaussian mixture distribution $p(x) = \sum_j \beta_j \normal(x; ~ \mu_j, \sigma_j^2)$ with $20$ randomly located mixture components shown in Figure~\ref{fig:gmm}(a), and draw $\{x_i\}_{i=1}^n$ from $p(x)$ itself. 
The MSEs for estimating $\E_ph$ with different $h(x)$ as the sample size $n$ increases are shown in Figure~\ref{fig:gmm}(b)-(d), 
where we generally find that our method tends to perform among the best. 

In Figure~\ref{fig:gmm2}(a), we study the performance of the algorithms on distributions with different Gaussianity, where we replace $p(x)$ with 
a series of distributions 
$p(x \mid \lambda)$ whose random variable is $(1-\lambda)x + \lambda \xi$
where $x\sim p$, $\xi \sim \normal(0,1)$ and $\lambda \in [0,1]$ controls the Gaussianity of $p(x \mid \lambda)$: it reduces to $p(x)$ when $\lambda=0$ and equals $\normal(0,1)$ when $\lambda = 1$.  
We observe that {\tt Stein} tends to perform the best when the distribution has high non-Gaussianity, 
but is suboptimal compared with {\tt Control Func} when the distribution is close to Gaussian.


In Figure~\ref{fig:gmm2}(b), we consider how the different algorithms scale to high dimensions by setting $p(x)$ to be the standard Gaussian distribution with increasing dimensions. 
We generally find that our {\tt Stein} tends to perform among the best under the different settings,
expect for low dimensional standard Gaussian under which {\tt Control Func} performs the best. 
The self-normalized versions of {\tt KDE}  and {\tt Control Func} can help to stabilize the algorithm in various cases, for example, 
{\tt KDE (Normalized)} significantly improves over {\tt KDE} in all the cases,  and 
{\tt Control Func (Normalized)} is significantly better than {\tt Control Func} in high dimensional cases as shown in Figure~\ref{fig:gmm2}(b). 

Figure~\ref{fig:gmm2lowerbound} shows the performance of our method with the non-negativity constraint ($w_i \geq 0$) replaced by $(w_i \geq -b)$ where 
$b$ is a positive number that takes different values. 
We find that the result of $w_i \geq 0$ generally performs the best when $n$ is small (e.g., $n < 1000$), but is slightly suboptimal when $n$ is large. 
Because the stability in the small $n$ case is more practically important than the large $n$ case,
given that the absolute difference on MSE would be negligible in the large $n$ region, 
we think enforcing $w_i \geq 0$ is a simple and good practical procedure. 

\begin{figure}[t]
   \begin{tabular}{ll}   
   \hspace{-.015\textwidth}
\includegraphics[height=0.2\textwidth]{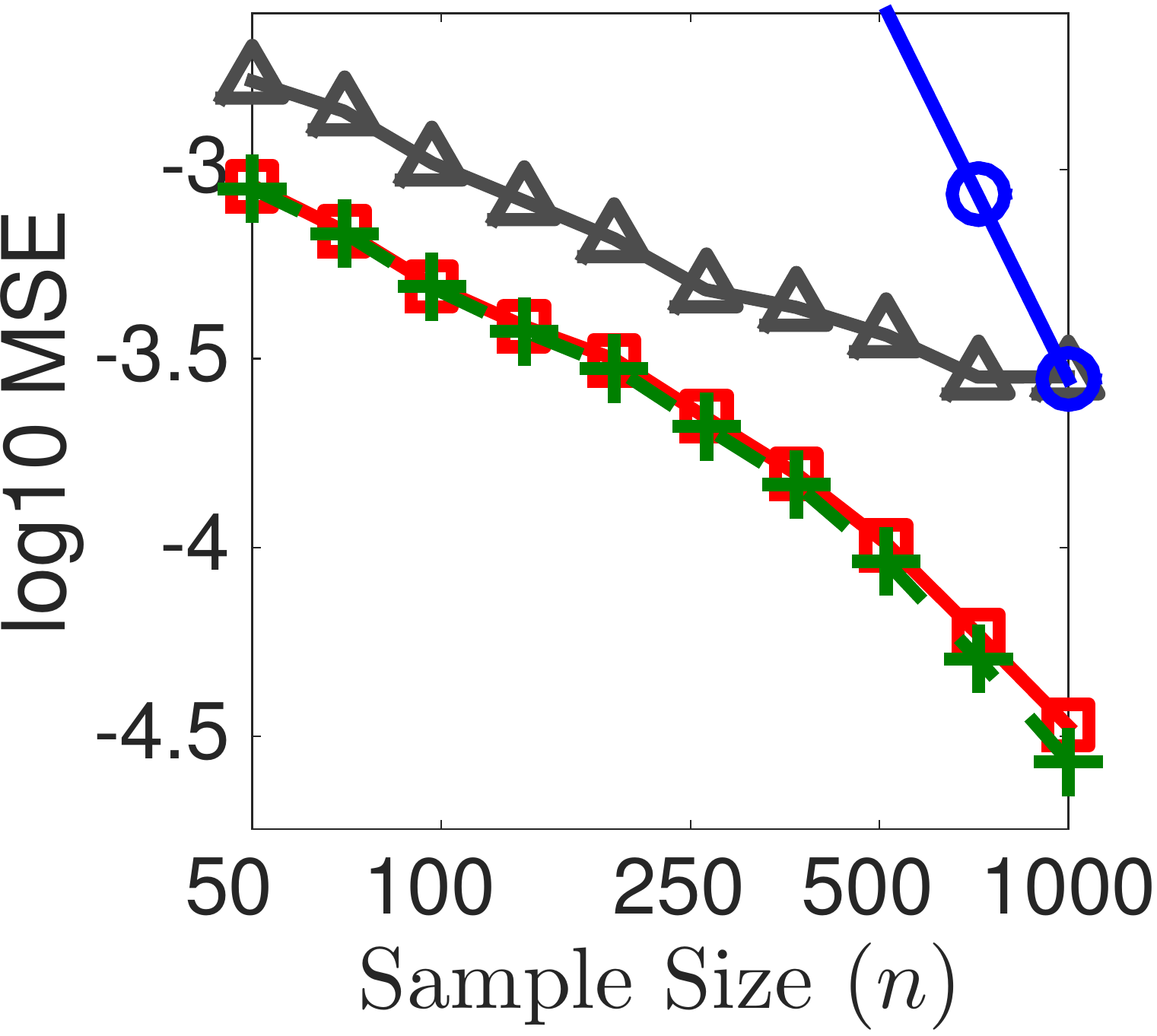}  & 
\hspace{.01\textwidth}   
\includegraphics[height=0.2\textwidth, trim={1.15cm 0 0 0},clip]{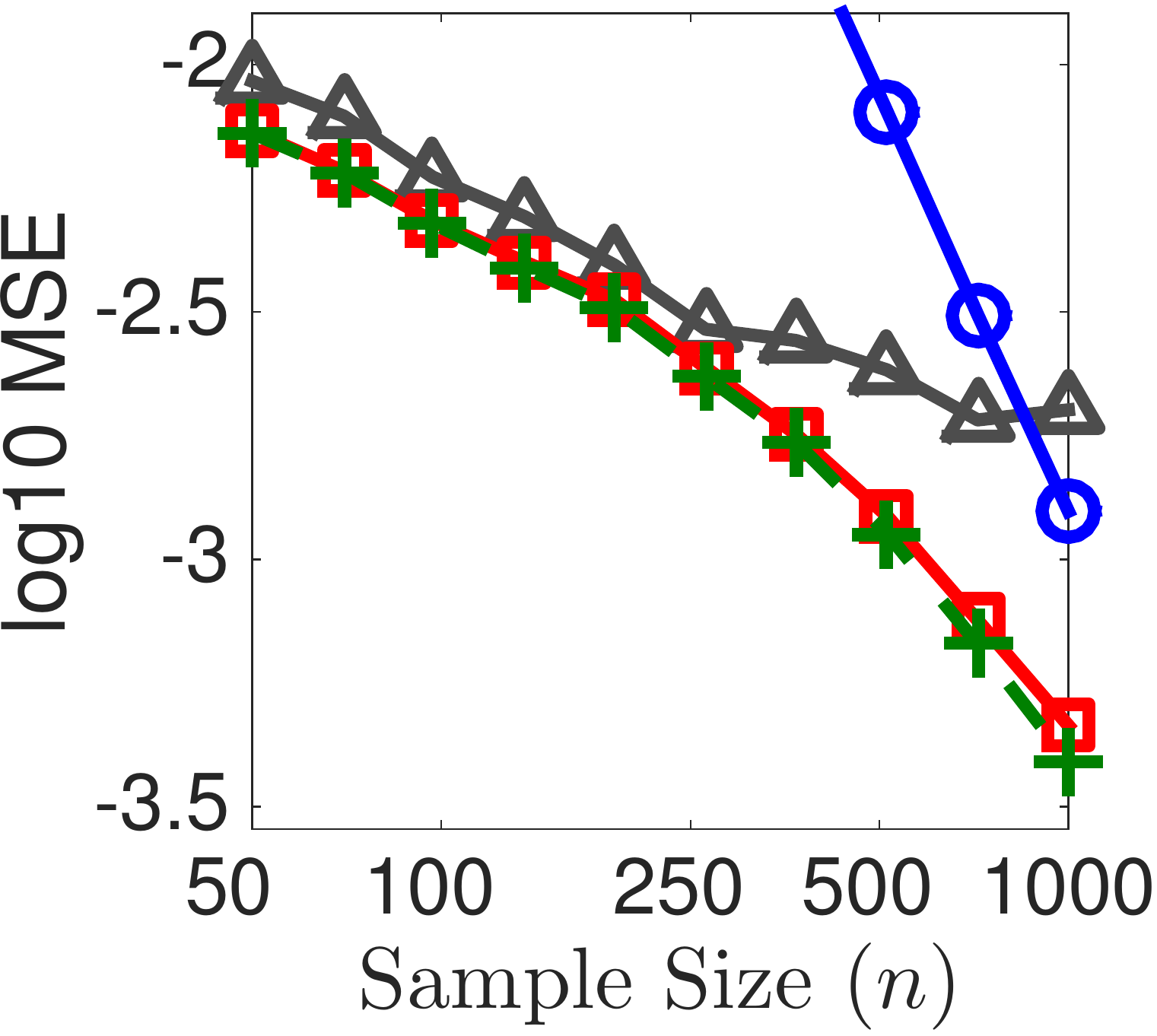}  
\hspace{-.27\textwidth}\raisebox{2.5em}{   \includegraphics[width=0.18\textwidth]{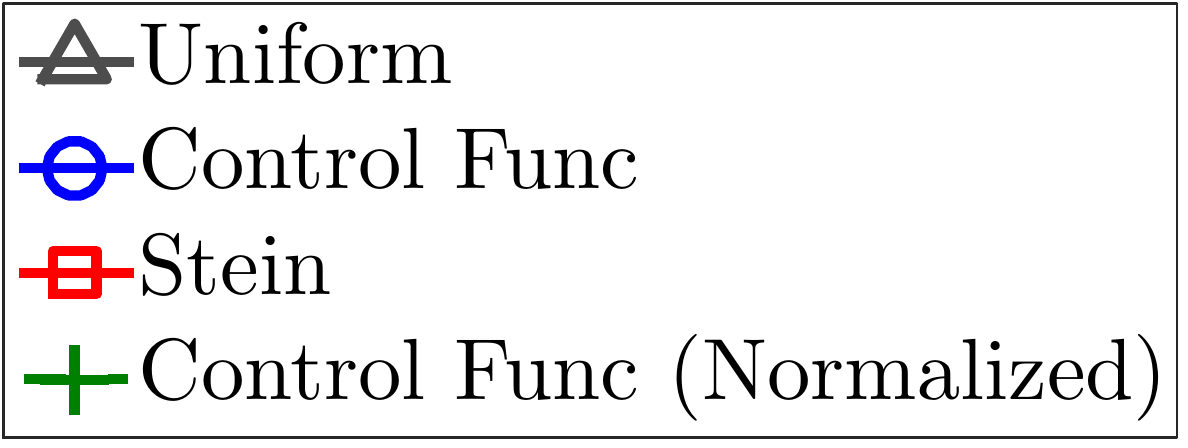} }
\\
{\footnotesize(a) Estimating $\E(x)$} 
& 
{\footnotesize(b) Estimating $\E(x^2)$} 
   \end{tabular}
   \caption{
\capfontsize   Results of Bayesian probit model 
   with simulated data. 
  We generate $\{x_i\}_{i=1}^n$ by simulating $n$ parallel chains of stochastic gradient Langevin Dynamics with a mini-batch size of $100$ for $100$-steps. 
 {\tt KDE} and {\tt KDE (Normalized)} perform significantly worse in this case, and are not show in the figure. 
 }
   \label{fig:probit1}
\end{figure}

\begin{figure}[t]
   \begin{tabular}{ll}  
   \hspace{-.015\textwidth}    
\includegraphics[height=0.2\textwidth]{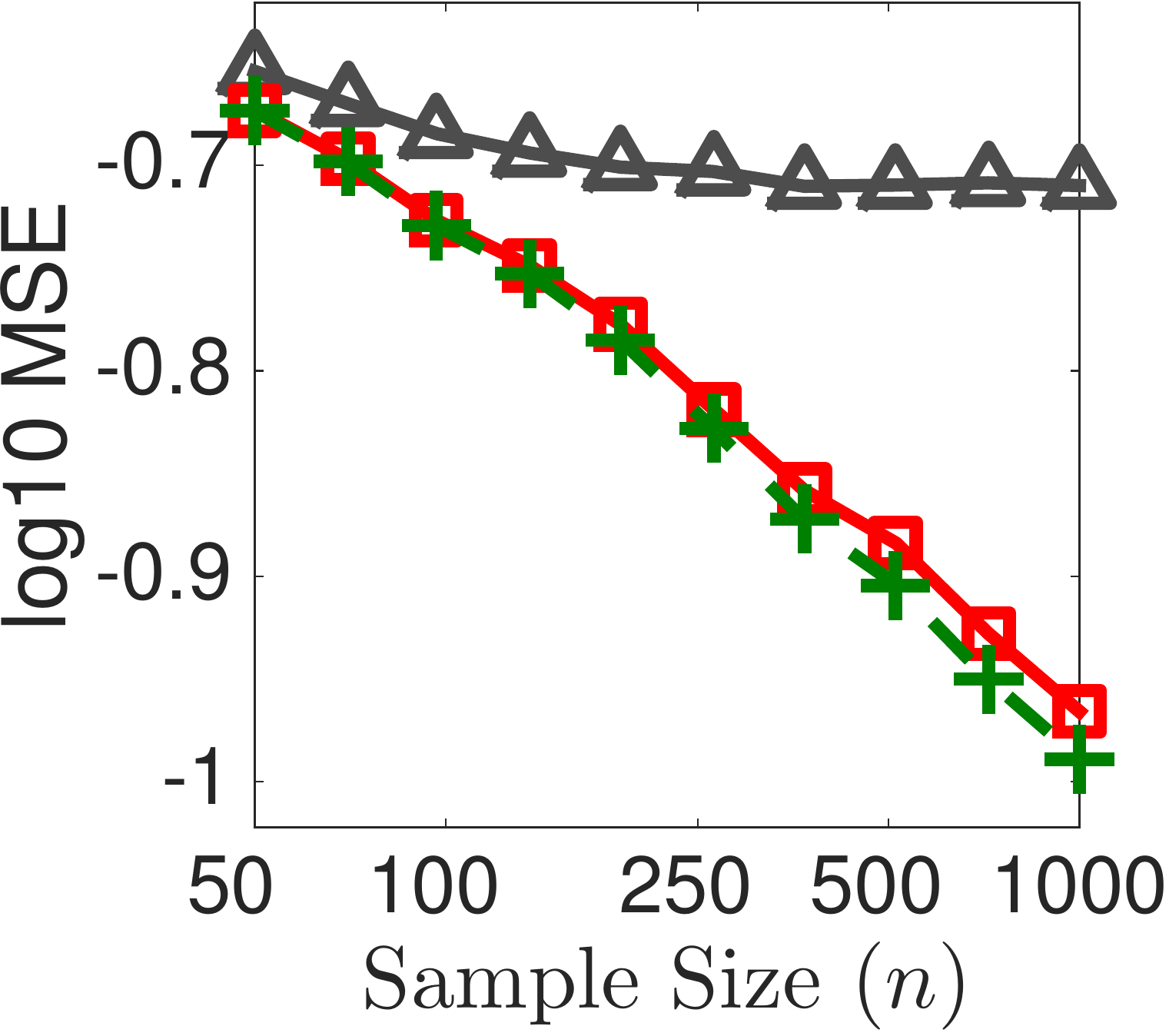}  & 
\hspace{.01\textwidth}   
\includegraphics[height=0.2\textwidth, trim={1.15cm 0 0 0},clip]{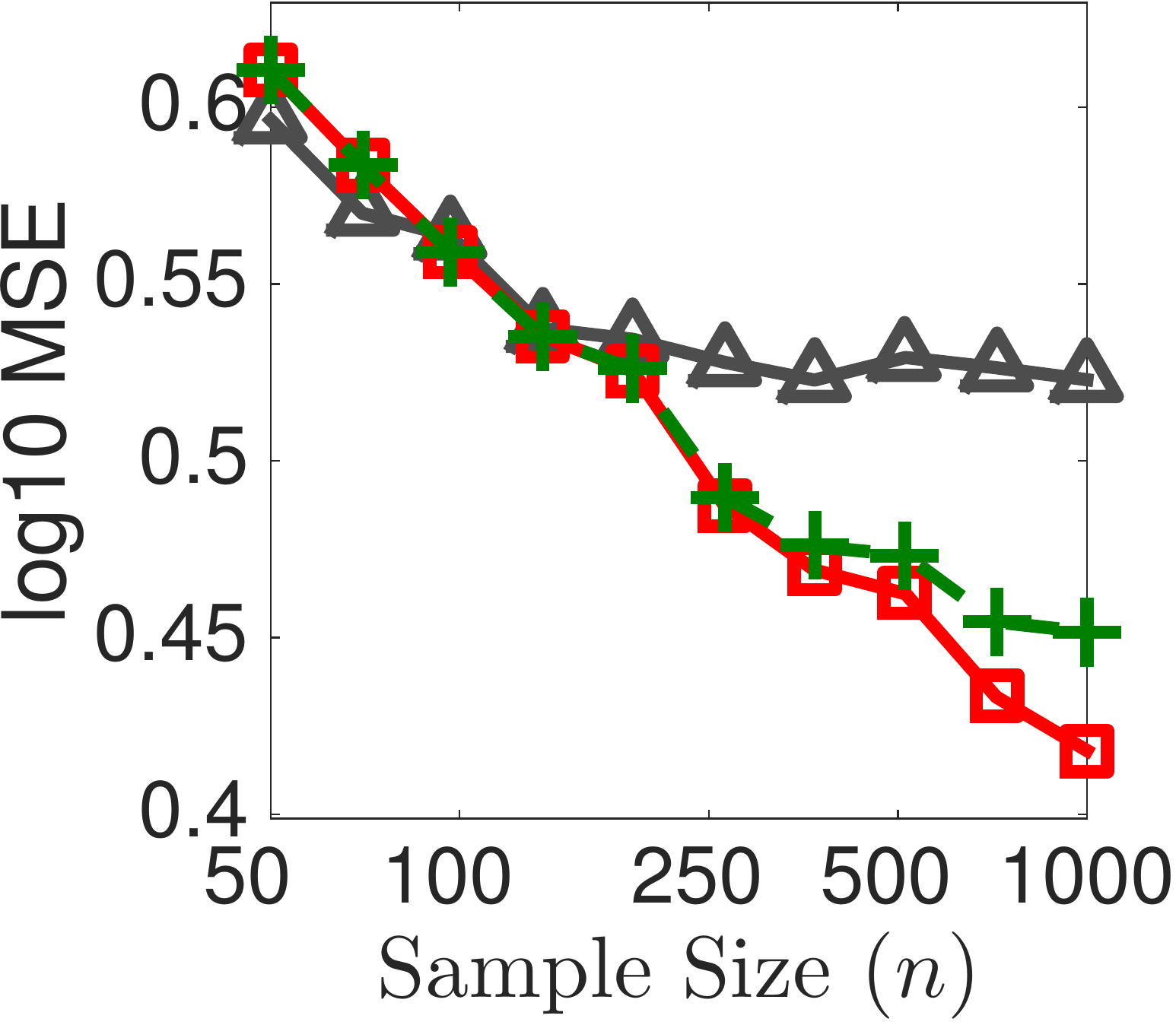}  
\hspace{-.28\textwidth}\raisebox{2.0em}{   \includegraphics[width=0.17\textwidth]{figures/Leg_NewsteinISstocLA0d00023342_covtypebinprobitn10000d54wtrueNaNmu0invS0d1_54_0d1_NVec50_10_1000_sVec1000.pdf} }
\\
{\footnotesize(a) Estimating $\E(x)$} 
& 
{\footnotesize(b) Estimating $\E(x^2)$} 
   \end{tabular}
   \caption{\capfontsize
   Results of Bayesian probit model 
   on the Covtype dataset.    
  We generate $\{x_i\}_{i=1}^n$ by simulating $n$ parallel chains of stochastic gradient Langevin Dynamics with a mini-batch size of $100$ for $1000$-steps. 
  The unnormalized {\tt Control Func}, as well as {\tt KDE} and {\tt KDE (Normalized)} perform significantly worse in this case, and are not show in the figure. 
 }
   \label{fig:probit2}
\end{figure}

\paragraph{Bayesian Probit Model}
We consider the Bayesian probit regression model for binary classification. Let $D = \{\chi_\ell, \zeta_\ell\}_{\ell=1}^N$ be a set of observed data with feature vector $\chi_\ell$ and binary label $\zeta_\ell\in \{0,1\}$. 
The distribution of interest is $p(x)\coloneqq p(D | x) p_0(x)$ with 
$$p( D | x) = \prod_{\ell=1}^N \big[ \zeta_\ell \Phi(x^\top \chi_\ell )  + (1-\zeta_\ell) ( 1-\Phi(x^\top \chi_\ell ) ) \big],$$
where $\Phi(\cdot)$ is the cumulative distribution function of the standard normal distribution, 
and $p_0(x) = \normal(x; 0, 0.1)$ is the prior. 

To test our method, we simulate $\{x_i\}_{i=1}^n$  
by running $n$ parallel chains of stochastic Langevin dynamics \citep{welling2011bayesian}. 
Since this method is an inexact MCMC, its stationary distribution should be different from the target distribution $p(x)$. 
As a result, directly averaging $\{x_i\}_{i=1}^n$ with uniform weights ({\tt Unif}) 
can give relatively poor results with convergence rate slower than the typical $\O(n^{-1/2})$ rate (see e.g., \citet{teh2014consistency}). The black-box weights can be used refine the result. 

Figure~\ref{fig:probit1} shows the result on a small simulated dataset with $100$ data instances and $10$ features. 
We can find that {\tt Stein} and {\tt Control Func (Normalized)} 
significantly improve the performance over {\tt Unif}. 
Interestingly, we find that the unnormalized {\tt Control Func}, as well as {\tt KDE} and {\tt KDE (normalized)} (not show in the figure) perform significantly worse in this case.   

Figure~\ref{fig:probit2} shows the result on the Forest Covtype dataset 
from the UCI machine learning repository \citep{Bache+Lichman:2013}; 
it has 54 features, and is reprocessed to get binary labels following \citet{collobert2002parallel}. 
For our experiment, we take the first 10,000 data points, so that it is feasible to evaluate the ground truth with 
No-U-Turn Sampler (NUTS) \citep{hoffman2014no}. 
We again find that  {\tt Stein} and {\tt Control Func (Normalized)} improves over the uniform weights, 
and the unnormalized {\tt Control Func} and {\tt KDE} and {\tt KDE (normalized)} again 
perform significantly worse and are not shown in the figure. 

 

\section{Conclusion}\label{sec:conclusion}
We propose a \emph{black-box importance sampling} method that calculates importance weights without knowing the proposal distribution, 
which also has the additional benefit of providing variance reduction.  
We expect our method provides a powerful tool for solving many difficult problems were previously intractable via importance sampling. 

{\small
\bibliographystyle{myunsrtnat}
\bibliography{bibrkhs_stein}
}

\onecolumn
\title{Appendix} 

\appendix
This document contains derivations and other supplemental information for ``Black-box Importance Sampling''.
\section{Kernelized Stein Discrepancy}
\todo{
Given RKHS $\H$ with kernel $\k(x,x')$, the maximum mean discrepancy (MMD) between two distributions with density $p(x)$ and $q(x)$ is defined as 
$$
\mathrm{MMD}_{\H}(q, p) =  \max_{f\in \H } \big\{  \E_q f - \E_p f ~~~s.t.~~~  ||f||_{\H}\leq 1 \}, 
$$
which can be shown to be equivalent to 
$$
\mathrm{MMD}_{\H}(q, p) = 
\E_{x,x'\sim p} [ \k(x, x') ]  - 2 \E_{x\sim p;y\sim q}[ \k(x, y) ]  + \E_{y,y'\sim q}[ \k(y, y') ].
$$
We show that kernelized discrepancy $\S(q, ~ p)$ is equivalent to $\mathrm{MMD}_{\Hp}(q, p)$, equipped with the $p$-Steinalized kernel $\kp(x,x')$. 
\begin{pro}
Assume \eqref{equ:Dhh} is true, we have
$$
  \S(q, ~ p)  = \mathrm{MMD}_{\Hp}(q, p). 
$$
\end{pro}
\begin{proof}
Simply note that $\E_{x' \sim p}[ \kp(x,x') ] = 0$ for any $x$, we have
$$
\mathrm{MMD}_{\Hp}(q, p)  = 
 \E_{x,x'\sim q}[\kp(x, x')] 
 = \S(q, ~ p).
 $$
\end{proof}
Similarly, we also have 
\begin{align*}
  \S(\{x_i, w_i\}_{i=1}^n, ~ p)  
  & = \mathrm{MMD}_{\Hp}(\{x_i, w_i\}, ~ p) \\
    & \coloneqq  \max_{f\in \H } \big\{  \sum_i w_i f(x_i) - \E_p f ~~~s.t.~~~ ||f||_{\H}\leq 1 \}. 
\end{align*}
}
\todo{The follow result is elementary, and also in \citet[Theorem 3][]{oates2014control}; see also, e.g., \citet{smola2007hilbert, bach2015equivalence}.  }
\myempty{
\begin{thm}\label{thm:mainbound}
Assume $h -\E_p h \in \Hp$, and $\sum_i w_i = 1$, we have 
$$
|  \sum_i w_i h(x_i)  -  \E_p h | \leq   || h - \E_p h ||_{\Hp} \cdot  \sqrt{  \S(\{w_i, x_i \}, ~ p) }
$$
\end{thm}
}
\begin{proof}[Proof of Proposition~\ref{pro:bound}]
Let $\tilde h(x) = h(x) - \E_p h$, we have
\begin{align*}
|\sum_i w_i  \tilde h(x_i) |   
& = |\sum_i w_i \la \tilde h, ~ \kp(\cdot, x_i) \ra_{\Hp}|\\
& = | \la \tilde  h, ~ \sum_i w_i  \kp(\cdot, x_i)\ra_{\Hp}| \\
& \leq || \tilde h ||_{\Hp} \cdot   ||  \sum_i w_i  \kp(\cdot, x_i)  ||_{\Hp} \\
& =  || \tilde h ||_{\Hp} \cdot  \sqrt{  \S(\{w_i, x_i \}, ~ p) }. 
\end{align*}
where we used Cauchy-Schwarz inequality and the fact that 
$||  \sum_i w_i  \kp(\cdot, x_i)  ||_{\Hp}^2 = \sum_{ij} w_i w_j \kp(x_i, x_j)  =  \S(\{w_i, x_i \}, ~ p) $.  
\end{proof}

\myempty{
\begin{lem}
$\S(\{x_i, w_i\}, ~ p)  = || \sum_i w_i \kp(\cdot, ~x_i) ||_{\Hp}$. 
\end{lem}
\begin{proof}
Define $f(\cdot) = \sum_i w_i \kp(\cdot, ~x_i) \in \Hp$, the result follows the fact that $|| f||_{\Hp}^2 = \sum_{ij} w_i \kp(x_i, x_j) w_j$.  
\end{proof}
}
\section{Convergence Rate}
We consider the error rate of our estimator $\sum_i \hat w_i(\vv x) h(x_i) $ with $\{\hat w_i(\vv x)\}$ given by the optimization in \eqref{equ:vw}, 
under the assumption that $\vv x = \{x_i\}_{i=1}^n$ is i.i.d. drawn from an (unknown) distribution $q(x)$. 
Based on the bound in Proposition~\eqref{pro:bound}, we can establish an error rate $\Op(n^{-\delta})$ by finding a set of oracle ``reference weights'' $\{w_{*i}(\vv x)\}$, as a function of $\vv x$, such that 
$\S(\{x_i, w_{*i}(\vv x)\}, ~ p) = \Op(n^{-2\delta})$, because
$$
|\sum_i \hat w_i(\vv x)  h(x_i)  - \E_p h |   
\leq  C_h \cdot  \sqrt{  \S(\{\hat w_i(\vv x), x_i \}, ~ p) } 
\leq  C_h \cdot  \sqrt{  \S(\{w_{*i}(\vv x), x_i \}, ~ p) } 
= \Op(n^{-\delta}) , 
$$
where $C_h = ||  h  - \E_p h||_{\Hp}$. This idea of using reference weights has been used in 
\citet{briol2015probabilistic} 
to study the convergence rate of Bayesian Monte Carlo. 

Section~\ref{sec:israte} proves the $\Op(n^{-1/2})$ rate using the typical importance sampling weights as the reference weight. 
Section~\ref{sec:controlrate} proves a better $\smallO{n^{-1/2}}$ rate by using a reference weight based on a control variates method constructed with an orthogonal basis estimator.  

\subsection{$\Op(n^{-1/2})$ Rate}
\label{sec:israte}
We use the typical importance sampling weight as a reference weight and establish $\Op(n^{-1/2})$ rate on the error of our estimator. 
\begin{ass}\label{ass:V}
Assume $p(x)/q(x) > 0 $ for $\forall x\in \X$ and $\E_{x\sim q} [\big(\frac{p(x)}{q(x)}\big)^2] < \infty$, 
 $\E_{x\sim q} ( | \frac{p(x)^2}{q(x)^2} \kp(x,x)|) < \infty$, and  
$\E_{x,x'\sim q} \big[\big( \frac{p(x)p(x')}{q(x)q(x')} \kp(x,x') \big)^2\big] < \infty$. 
\end{ass}

\begin{lem} 
Assume $\{x_i\}_{i=1}^n$ is i.i.d. drawn from $q(x)$ 
$$w_i^* = \frac{1}{Z} p(x_i) / q(x_i), ~~~ Z = \sum_i p(x_i) / q(x_i),$$ 
then under Assumption~\ref{ass:V} we have 
\begin{align*}
\S(\{w_i^*, x_i\}, ~ p) =  \Op(n^{-1}). 
\end{align*}
\end{lem}
\begin{proof}
Define $v_i^*(x_i) = \frac{1}{n} p(x_i)/ q(x_i)$, and
\begin{align*}
\S(\{v_i^*, x_i\}, ~ p) =   \frac{1}{n^2} \sum_{ij}  \frac{ p(x_i)}{q(x_i)} \frac{p(x_j)}{q(x_j)} \kp(x_i, x_j), 
\end{align*}
then $\S(\{v_i^*, x_i\}, ~ p)$  is a degenerate V-statistic since by \eqref{equ:Dhh} we have 
$$
\E_{x'\sim q}  [\frac{ p(x)}{q(x)} \frac{p(x')}{q(x')} \kp(x', x')] 
= 
\frac{ p(x)}{q(x)} \E_{x'\sim p}  [ \kp(x_i, x_j)]
= 0, ~~~~~ \forall x \in \X 
$$ 
then we have \citep[see e.g.,][]{lee90ustatistics}
$$
\S(\{v_i^*, x_i\}, ~ p) =  \Op(n^{-1}). 
$$
In addition, note that $\sum_{i=1}^n v_i^* =  1 + \Op(n^{-1/2})$, we have 
$$
\S(\{w_i^*, x_i\}, ~ p)  = \frac{\S(\{v_i^*, x_i\}, ~ p)}{(\sum_i v_i^*)^2}  = \Op(n^{-1}). 
$$
\end{proof}

\begin{thm}\label{thm:is}
Assume $\{x_i\}$ is i.i.d. drawn from $q(x)$, 
and $\{\hat w_i(\vv x)\}$ is given by \eqref{equ:vw}, 
then under Assumption \ref{ass:V}, we have 
$$ 
\sum_{i=1}^n \hat w_i (\vv x) h(x_i) - \E_p h = \Op({n^{-1/2}}).$$
\end{thm}
\begin{proof}
Simply note that  
$$
\S(\{ \hat w_i, x_i\}_{i=1}^n, ~ p) \leq \S(\{w_i^*, x_i\}_{i=1}^n, ~ p)  = \Op({n^{-1}}), 
$$
and combining with Proposition~\ref{pro:bound} gives the result. 
\end{proof}

\subsection{${\scriptstyle\Op}(n^{-1/2})$ Rate}
\label{sec:controlrate}

We prove Theorem~\ref{thm:mainb0} that shows an $\smallO{n^{-1/2}}$ rate for our estimator. Our method is based on constructing a reference weight by using a two-fold control variate method based on the first $L$ orthogonal eigenfunctions $\{\phi_\ell\}$  of kernel $\kp(x,x')$. 
\todo{We construct the weights using a variant of the control functional method in \citet{oates2014control}, but with a orthogonal series estimator constructing the estimator $\hat h(x).$}

We first re-state the assumptions made in Theorem~\ref{thm:mainb0}. 
\begin{ass}\label{ass:kr}
1. Assume $\kp(x,x') $ has the following eigen-decomposition 
$$\kp(x,x') = \sum_\ell \lambda_\ell \phi_\ell(x) \phi_\ell(x'),$$ 
where $\lambda_\ell$ are the positive eigenvalues sorted in non-increasing order,  
and $\phi_\ell$ are the eigenfunctions orthonormal w.r.t.  distribution $p(x)$, that is, 
$$\E_{p}[\phi_\ell \phi_{\ell'}] \overset{\mathrm{def}}{=}  \int p(x) \phi_\ell(x) \phi_{\ell'}(x) dx = \ind[\ell = \ell'].$$ 

2. $\trace(\kp(x,x')) = \sum_{\ell=1}^{\infty} \lambda_\ell <  \infty$.

3. $\var_{x \sim q}[w_*(x)^2 \phi_\ell(x) \phi_{\ell'}(x)] \leq M$ for all $\ell$ and $\ell'$, where $w_*(x)= p(x)/q(x)$. 

4. $| \phi_\ell(x) |^2 \leq M_2$, and 
$w_*(x) \overset{def}{=} p(x)/q(x) \leq M_3$ for any $x \in \X$. \todo{[Is it too strong?]}
\end{ass}
\todo{[Do I also need?] $\X$ is a compact open subset in $\R^d$\red{[necessary?]}. }



The following is an expended version of Theorem~\ref{thm:mainb0}. 
\begin{thm}\label{thm:mainb}
Assume $\{x_i\}_{i=1}^n$ is i.i.d. drawn from $q(x)$, and $\hat w_i$ is calculated by 
$$\hat {\vv w} = \argmin_{\vv w} \vv w \KP \vv w, ~~~ s.t. \sum_i w_i = 1, ~~~ w_i \geq 0 ,$$
and $h- \E_p h \in \Hp$.
Under Assumption~\ref{ass:kr}, we have
\begin{align*}
&  \E_{\vv x\sim q}(| \sum_i \hat w_i h(x_i) - \E_p h|^2)= \Op(
 \frac{1}{n}  \gamma(n)), \\ 
 &   \text{where}     ~~~~~ \gamma(n) = \min_{\Ln\in \mathbb N^+} \big\{  \frac{M_3}{2} \R(\Ln)   ~+~    \frac{M_4}{2}  \frac{\Ln}{n}   +   M_f n (n+2) \exp(-\frac{n}{L^2 M_0})   \big\}, 
\end{align*}
where $\mathbb N^+ $ is the set of positive integers, and $\R(\Ln) =  \sum_{\ell > \Ln} \lambda_\ell$ is the residual of the spectrum, and $M_4 = 2 M_3 M \trace(\k_p)$. 
and $M_f =  \trace(k_p(x,x')) M_2 $ and  $M_0 = \max(M_2^2 M_3, ~  M_3^2(M_2 M_3 + \sqrt{2})^2)$.  
\end{thm}

\paragraph{Remark} To see how Theorem~\ref{thm:mainb} implies Theorem~\ref{thm:mainb0}, 
we just need to observe that we obviously have $\gamma(n) \geq 2 M_3 \frac{b}{n},$ and
$\gamma(n) = \smallO{1}$ by taking $L = n^{1/4}$.

%
\myempty{
\begin{thm}
Assume $\{x_i\}_{i=1}^n$ is i.i.d. drawn from $q(x)$, and $\hat w_i$ is calculated by 
$$\hat {\vv w} = \argmin_{\vv w} \vv w \KP w, ~~~ s.t. \sum_i w_i = 1, ~~~ w_i \geq 0 ,$$
and $h- \E_p h \in \Hp$.
Under Assumption~\ref{ass:kr}, we have
\begin{align*}
&  \E_{\vv x\sim q}(| \sum_i \hat w_i h(x_i) - \E_p h|^2)= \Op(
 \frac{1}{n}  \gamma(n)), \\ 
  \text{where}  & ~~~~~ \gamma(n) = \min_{\Ln\in \mathbb N} \big\{ \R(\Ln)   ~+~    \frac{b}{n} \Ln  ~~~~~ s.t.~~~~~ \Ln = \smallO{n/\log n}  \big\}, 
\end{align*}
where $\mathbb N^+$ is the set of natural numbers, and $\R(\Ln) =  \sum_{\ell > \Ln} \lambda_\ell$ is the residual of the spectrum, and $b = 2 \max_{\ell'}  \{ \sum_\ell  \lambda_{\ell} \rho_{\ell\ell'} \} \leq  2 M \trace(\k_p)$. 
\end{thm}
}
%
%
%
%
Based on Proposition~\ref{pro:bound},  to prove Theorem~\ref{thm:mainb} we just need to show that for any $\vv x =\{x_i\}_{i=1}^n$, there exists a set of positive and normalized weights $\{w_{i}^+(\vx )\}$, as a function of $\vv x$, such that 
$$
\E_{\vv x \sim q} [\S(\{ w_{i}^+ (\vx), x_i \}, ~ p)]  =  \Op(\frac{\gamma(n)}{n}). 
$$
In the sequel, we construct such a weight based on a control variates method which uses the top eigenfunctions $\phi_\ell$ as the control variates. Our proof includes the following steps: 
\begin{enumerate}
\item Step 1: Construct a control variate estimator based on the orthogonal eigenfunction basis, and obtain the corresponding weights $\{w_i(\vv x)\}$. 
\item Step 2: Bound $\E_{\vv x \sim q} [\S(\{ w_{i} (\vx), x_i \}, ~ p)]$. 
\item Step 3. Construct a set of positive and normalized weights by $w_i^+(\vv x) = \frac{\max(0, w_i(\vv x))}{\sum_i \max(0, w_i(\vv x))}$, and establish the corresponding bound. 
\end{enumerate}
\begin{proof}[Proof of Theorem~\ref{thm:mainb}]
Combine the bound in Lemma~\ref{lem:sss} and Lemma~\ref{lem:boundwp} below. 
\end{proof}

We note that the idea of using reference weights was used in \citet{briol2015probabilistic} to establish the convergence rate of Bayesian Monte Carlo. 
Related results is also presented in \citet{bach2015equivalence}. The main additional challenge in our case is to meet the non-negative and normalization constraint (Step 3); this is achieved by showing that the $\{w_i(\vx)\}$ constructed in Step 2 is non-negative with high probability, and their sum approaches to one when $n$ is large, and hence $\{w_i^+(\vx)\}$ is not significantly different from $\{w_i(\vx)\}$. 

Note that if we discard the non-negative and normalization constraint (Step 3), the error bound would be $\Op(\gamma_0(n)n^{-1})$, where 
$$\gamma_0(n)= \min_{L\in \mathbb N^+}\{ {2 M_3}\R(L) + 2 M_4  \frac{\Ln}{ n} \},$$
 as implied by Lemma~\ref{lem:sss}. 
Therefore, the third term in $\gamma(n)$ is the cost to pay for enforcing the constraints. 
However, this additional term does not influence the rate significantly once $\R(L) = \sum_{\ell > L} \lambda_\ell$ decays sufficiently fast. 
For example,  
when $\R(L) = \Op(L^{-\alpha})$ where $\alpha > 1$,  both $\gamma(n)$ and $\gamma_0(n)$ equal  $\Op(n^{-1 + 1/(\alpha+1)})$; when $\R(L) = \Op(\exp(-\alpha L))$ with $\alpha > 0$, both $\gamma(n)$ and $\gamma_0(n)$ equal $\Op(\frac{\log n}{n})$. 
An open question is to derive upper bounds for the decay of eigenvalues $\R(L)$ for given $p$ and $k(x,x')$, so that actual rates can be determined. 



\myempty{
\begin{thm}
Under Assumption~\ref{ass:kr}, and $\vv x= \{x_i\}_{i=1}^n$ is i.i.d. drawn from $q(x)$, then there exist a set of weights $\{ w_i(\vv x)\}$, as a function of $\vv x$, such that 
$$
\E_{\vv x \sim q} [\sum_{i} w_i(\vv x) h(x_i)] = \frac{2}{n} \sum_\ell \bigg [\lambda_\ell (1 - \tau_\ell)^2 + \frac{2}{n} b_\ell \tau_\ell^2 \bigg]. 
$$
\end{thm}
}

\subsection*{Step 1: Constructing the weights}
We first construct a set of unnormalized, potentially negative reference weights, 
by using a two-fold control variates method based on the orthogonal eigenfunctions $\{\phi_\ell\}$  of kernel $\kp(x,x')$. 
Assume $n$ is an even number, and we partition the data $\{x_i\}_{i=1}^n$ into two parts $\D_0 =\{1,\ldots, \frac{n}{2}\}$ and $\D_1 = \{\frac{n}{2}+1, \ldots n\}$. 
For any $h \in \Hp$, we have $\E_ph = 0$ by \eqref{equ:Dhh}, and 
$$
h(x) = \sum_{\ell=1}^{\infty} \beta_\ell \phi_\ell( x),~~~~~~~~ \beta_\ell = \E_{x\sim p} [h(x) \phi_\ell(x)]. 
$$
We now construct  an orthogonal series estimator $\hat h(x)$ for $h(x)$ based on $\vv x _{\D_0}$, 
\begin{align}
\hat h_{\D_0}(x) & = \sum_{\ell=1}^{\Ln} \hat\beta_{\ell, 0} \phi_\ell(x), ~~~~~\text{where}~~~~~\hat\beta_{\ell,0} = \frac{2}{n}\sum_{i \in \D_0} h(x_i) \phi_\ell(x_i) \frac{p(x_i)}{q(x_i)},
\end{align}
where we approximate $\beta_\ell$ with an unbiased estimator $\hat \beta_{\ell,0}$ since
$$
\E_{x\sim q} [\hat\beta_{\ell,0}] = \E_{x\sim q}[h(x) \phi_\ell(x) \frac{p(x)}{q(x)}] = \int p(x) h(x)\phi_\ell(x) dx = \beta_\ell. 
$$
We also truncate at the $\Ln$th basis functions to keep $\hat h_{\D_0}(x)$ a smooth function, as what is typically done in orthogonal basis estimators. 
 We will discuss the choice of $\Ln$ later. 
Based on this we define a control variates estimator: 
\begin{align*}
\hat Z_0[h] = \frac{2}{n}\sum_{i\in \D_1} [w_*(x_i)(h(x_i) - \hat h_{\D_0}(x_i))],
\end{align*}
which gives an unbiased estimator for $\E_p h= 0$ because 
$$
\E_{\vv x\sim q}(\hat Z_0[h]) 
= \int q(x) \frac{p(x)}{q(x)} (h(x) - \hat h_{\D_0}(x_i)) dx 
= \E_{x\sim p} h  - \E_{\vv x_{\D_0} \sim q} \big[ \E_{x\sim p}[ \hat h_{\D_0}(x) \cd \vv x_{\D_0}] \big] = 0, 
$$
where the last step is because $\E_{x\sim p} [\hat h_{\D_0}(x) \mid \vv x_{\D_0}] =\sum_{\ell=1}^\Ln  \hat \beta_{\ell,0}  \E_{x\sim p}  [\phi_\ell(x)] = 0$. 
Switching $\D_0$ and $\D_1$, we get another estimator 
$$
\hat Z_1[h] = \frac{2}{n}\sum_{i\in \D_0} [w_*(x_i)(h(x_i) - \hat h_{\D_1}(x_i))]. 
$$
Averaging them gives 
$$
\hat Z[h] = \frac{\hat Z_0[h]  + \hat Z_1[h]}{2}. 
$$
\begin{lem}\label{lem:defw}
Given $\hat Z[h]$ defined as above, for any $h \in \Hp$, we have 
\begin{align*}
\hat Z [h] = \sum_{i=1}^n w_i(\vv x) h(x_i), 
&&\text{with}&&
w_i(\vv x) = \begin{cases}
 \frac{1}{n} w_*(x_i)  - \frac{2}{n^2} \sum_{j\in \D_1} w_*(x_i)  w_*(x_j) \k_{\Ln}(x_j, x_i),  & \text{$\forall i\in \D_0$} \\[.5em]
 \frac{1}{n} w_*(x_i)  - \frac{2}{n^2}  \sum_{j\in \D_0} w_*(x_i)  w_*(x_j)  \k_{\Ln}(x_j, x_i),  & \text{$\forall i\in \D_1$}  
\end{cases}
\end{align*}
where $w_*(x)=p(x)/q(x)$ and $\k_{\Ln}(x, x') = \sum_{\ell=1}^{\Ln} \phi_\ell(x)\phi_\ell(x')$. 
\end{lem}
\begin{proof}
We have 
\begin{align*}
\hat Z_0[h] 
&= \frac{2}{n}\bigg[\sum_{i\in \D_1} w_*(x_i) \big( h(x_i)  - \hat h_{\D_0}(x_i) \big ) \bigg] \\
&= \frac{2}{n}\bigg[\sum_{i\in \D_1} w_*(x_i)\big (h(x_i) - \sum_{\ell=1}^{\Ln}  \hat \beta_{\ell,0} \phi_\ell(x) \big ) \bigg] \\
&= \frac{2}{n}\bigg[\sum_{i\in \D_1} w_*(x_i) \big(h(x_i) - \frac{2}{n}  \sum_{\ell=1}^\Ln  \sum_{j\in \D_0} h(x_j)  w_*(x_j) \phi_\ell(x_j) \phi_\ell(x_i) \big ) \bigg]\\
&= \frac{2}{n} \sum_{i\in \D_1} w_*(x_i) h(x_i)  ~ - ~  \frac{4}{n^2} \sum_{j\in \D_0} \sum_{i\in \D_1}  h(x_j)  w_*(x_j)   w_*(x_i) \sum_{\ell=1}^\Ln  \phi_\ell(x_j) \phi_\ell(x_i) \\ 
&= \frac{2}{n} \sum_{i\in \D_1} w_*(x_i) h(x_i)  ~ - ~  \frac{4}{n^2} \sum_{j\in \D_0} \sum_{i\in \D_1}  h(x_j)  w_*(x_j) w_*(x_i)  \k_\Ln(x_i, x_j) \\ 
& \overset{\mathrm{def}}{=} \sum_{i=1}^n w_{i,0}  h(x_i),
\end{align*}
where 
\begin{align}
w_{i,0} = 
\begin{cases}
 - \frac{4}{n^2} \sum_{j \in \D_1} w_*(x_i) w_*(x_j) \k_{\Ln}(x_j, x_i)   & \text{$\forall i\in \D_0$} \\
\frac{2}{n} w_*(x_i)   & \text{$\forall i \in \D_1$} 
\end{cases}
\end{align}
We can derive the same result for $\hat Z_1[h]$ and averaging them would gives the result. 
\end{proof}

\subsection*{Step 2: Calculating $ \E_{\vv x\sim q}(\S(\{x_i, w_i(\vv x)\}, ~ p ))$} 
\begin{lem} \label{lem:sss}
Under Assumption~\ref{ass:kr}, for the weights $\{w_i(\vv x)\}$ defined in Lemma~\ref{lem:defw}, we have
$$
\E_{\vv x\sim q} [ \S(\{x_i, w_i(\vv x)\}, ~ p ]
\leq 
 \frac{2}{n} [ M_3 \R(\Ln)   ~+~    M_4 \frac{\Ln}{n}  ]
$$
where $M_3$ is the upper bound of $p(x)/q(x), ~ \forall x\in \X$ and $\R(\Ln) =  \sum_{\ell > \Ln} \lambda_\ell$ and $M_4 = 2 M_3\max_{\ell'}  \{ \sum_\ell  \lambda_{\ell} \rho_{\ell\ell'} \} \leq  2 M_3 M \trace(\k_p)$. 
\end{lem}
\begin{proof}
First, for any $h \in \Hp$ (such that $\E_p[h] = 0$), we have 
\begin{align*}
& \E_{\vv x \sim q} \left [ \hat Z_0[h]^2 \right ]   \\
& = \E_{\vv x\sim q} \left [\bigg (\frac{2}{n}\sum_{i \in \D_1} w_*(x_i)( h(x_i) - \hat h_{\D_0}(x_i))  \bigg)^2\right]   \\
& = \frac{4}{n^2} \E_{\vv x_{\D_0}\sim q} \bigg\{ \sum_{i \in \D_1} \E_{x_i \sim q} \left [ w_*(x_i)^2( h(x_i) - \hat h_{\D_0}(x_i) )^2 \right ]  
\\
& ~~~~~~~~~~~~~~~~~~~~~+   \sum_{i\neq j; i,j\in \D_1}  \E_{x_i, x_j\sim q} \left [ w_*(x_i) (h(x_i) - \hat h_{\D_0}(x_i)) w_*(x_j) (h(x_j) - \hat h_{\D_0}(x_j)) \right ]      \bigg\}\\
& = \frac{4}{n^2}  \E_{\vv x_{\D_0}\sim q}  \bigg\{\sum_{i \in \D_1} \E_r \left[ ( h(x_i) - \hat h_{\D_0}(x_i) )^2 \right ]  ~ + 
\sum_{i\neq j; i,j\in \D_1}  \E_p \left [(h(x_i) - \hat h_{\D_0}(x_i))  (h(x_j) - \hat h_{\D_0}(x_j)) \right ]     \bigg\}  \\
& = \frac{2}{n}  \E_{\vv x_{\D_0}\sim q}  \bigg\{  \int  \frac{p(x)^2}{q(x)} ( h(x) - \hat h_0(x) )^2 dx    \bigg\}  ~~~~~~~~\text{\small (because $\E_p h = \E_p \hat h = 0)$}   \\
& \leq \frac{2 M_3 }{n}  \E_{\vv x_{\D_0}\sim q}  \bigg\{ \E_p [ ( h(x) - \hat h_0(x) )^2 ]    \bigg\}   ~~~~~~~~~~~~~~~~\text{\small (because $p(x)/q(x)\leq M_3$ by assumption)}   \\
 & = \frac{2 M_3}{n}   \E_{\vv x_{\D_0}\sim q}  \bigg\{   \sum_{\ell > \Ln} \beta_\ell^2 + \sum_{\ell < \Ln}  (\beta_\ell - \hat \beta_{\ell, 0})^2    \bigg\} \\
  & = \frac{2 M_3}{n}   \bigg\{ \sum_{\ell > \Ln} \beta_\ell^2 + \sum_{\ell < \Ln}   \var_{\vv x_{\D_0}\sim q}(\hat \beta_{\ell, 0})   \bigg\}  ~~~~~~~~~~~~~~~~~~~~\text{\small (because $\E_{\vv x_{\D_0}\sim q} [\hat \beta_{\ell,0} ]  =\beta_{\ell}$)}\\
  & = \frac{2 M_3}{n}   \left [    \sum_{\ell > \Ln} \beta_\ell^2 + \frac{2}{n} \sum_{\ell < \Ln}  \var_{x \sim q}[w_*(x)\phi_\ell(x) h(x)]      \right ]. 
    \end{align*}
    We can derive the same result for $\hat Z_1 [h]$ and hence  
  \begin{align*}
\E_{\vv x\sim q} \big[\hat Z[h]^2\big] 
& \leq \frac{1}{2}(\E_{\vv x\sim q} \big[\hat Z_0[h]^2\big]  + \E_{\vv x\sim q} \big[\hat Z_1[h]^2\big] )  \\
& = \frac{2M_3}{n}   \left [   \sum_{\ell > \Ln} \beta_\ell^2 + \frac{2}{n} \sum_{\ell < \Ln}  \var_{x \sim q}[w_*(x)  \phi_\ell(x) h(x)]    \right]. 
 \end{align*}
 Taking $h(x) = \phi_{\ell'}(x)$ for which we have $\beta_{\ell} = \ind[\ell = \ell']$, we get
  $$
 \E_q \big[\hat Z[\phi_{\ell'}]^2\big] 
 \leq 
 \begin{cases}
 \frac{4 M_3}{n^2} \sum_{\ell < \Ln}  \var_{x \sim q}[w_*(x) \phi_\ell(x) \phi_{\ell'}(x)]  & \text{if $\ell'  \leq \Ln$} \\[.1em]
 \frac{2 M_3}{n} +  \frac{4 M_3}{n^2} \sum_{\ell < \Ln}  \var_{x \sim q}[w_*(x) \phi_\ell(x) \phi_{\ell'}(x)]  & \text{if $\ell'  >  \Ln$} . 
\end{cases} 
 $$
Define  
$\rho_{\ell\ell'}=\var_{x \sim q}[w_*(x) \phi_\ell(x) \phi_{\ell'}(x)]$
and we have $\rho_{\ell\ell'}\leq M$ by Assumption~\ref{ass:kr}. 
We have 
\begin{align*}
\E_{\vv x\sim q}[\S(\{x_i, w_i(\vv x)  \}, p) ]
& =\E_{\vv x\sim q}[\sum_{i,j=1}^n w_i(\vv x) w_j(\vv x)  \kp(x_i, x_j) ] \\
& =\E_{\vv x\sim q}[\sum_{i,j=1}^n w_i (\vv x) w_j(\vv x)  \sum_{\ell=1}^\infty \lambda_\ell \phi_\ell(x_i) \phi_\ell(x_j) ]\\
& =\sum_{\ell} \lambda_\ell \E_{\vv x\sim q}[ (\sum_{i=1}^n w_i(\vv x)   \phi_\ell(x_i))^2   ]\\
& =\sum_{\ell} \lambda_\ell \E_{\vv x\sim q}[ \hat Z[\phi_\ell]^2   ]\\
& \leq \frac{2 M_3}{n} [  \sum_{\ell > \Ln} \lambda_\ell   ~+~    \frac{2}{n} \sum_{\ell=1}^\infty  \lambda_\ell   \sum_{\ell' < \Ln} \rho_{\ell\ell'}  ] \\
& \leq \frac{2 }{n} [  M_3 \sum_{\ell > \Ln} \lambda_\ell   ~+~     M_4 \frac{\Ln}{n}  ] , 
\end{align*}
where $M_4 = 2 M_3 \max_{\ell'}  \{ \sum_\ell  \lambda_{\ell} \rho_{\ell\ell'} \} \leq  2M_3  M \trace(\k_p)$. 
\end{proof}

\subsection*{Step 3: Meeting the Non-negative and Normalization Constraint}
The weights defined in \eqref{lem:defw} is not normalized to sum to one, and may also have negative values. 
To complete the proof, we define a set of new weights, 
$$
 w^+_i(\vv x) =\frac{ \max(0, w_i(\vv x)) }{\sum_i \max(0, w_i(\vv x))}. 
$$
We need to give the bound for $\S(\{x_i, w^+_i(\vv x) \},~ p)$ based on the bound of $\Op( \S(\{x_i, w_i(\vv x)\},~ p)  )$. 
The key observation is that we have $\sum_{i=1}^n w_i(\vv x) \pto 1 $ and $w_i(\vv x) \geq 0$ with high probability for the weights given by in Lemma~\ref{lem:defw}. 
\begin{lem}\label{lem:ineq}
For the weights $\{w_i(\vx)\}$ defined in Lemma~\ref{lem:defw}, under Assumption~\ref{ass:kr}, we have

i). When $\vx = \{x_i\}_{i=1}^n \sim q$, we have 
\begin{align} 
\label{equ:wineg}
\prob[ w_i(\vx)  <  0]  \leq \exp( - \frac{n}{ L M_2^2 M_3^2}), ~~~~~ \text{for~~ $\forall i\leq n$.} 
\end{align}
ii).  
We have $\E_{\vx \sim q} [\sum_i w_i(\vx) ]  = 1$. Assume $L \geq 1$, we have 
\begin{align}\label{equ:sineq}
\prob(S < 1 - t)  \leq 2 \exp(- \frac{n}{L^2 M_s}) &&
\text{where} &&
M_s = M_3^2( M_2 M_3 + \sqrt{2})^2/4, 
\end{align}
\end{lem}
\begin{proof}
i). 
Recall that 
\begin{align*}
w_i(\vx) = 
\begin{cases}
   \frac{1}{n} w_*(x_i)  - \frac{2}{n^2} \sum_{j\in \D_1}   w_*(x_i)  w_*(x_j) \k_{\Ln}(x_j, x_i),  & \text{$\forall i\in \D_0$} \\[.5em]
   \frac{1}{n} w_*(x_i)  - \frac{2}{n^2} \sum_{j\in \D_0} w_*(x_i) w_*(x_j)    \k_{\Ln}(x_j, x_i),  & \text{$\forall i\in \D_1$}  . 
\end{cases}
\end{align*}
We just need to prove \eqref{equ:wineg} for $i\in \D_0$. Note that
\begin{align*}
w_i(\vx)  = \frac{1}{n} w_*(x_i) \big[1  -  T \big],  &&
\text{where} &&
T= \frac{2}{n}  \sum_{j\in \D_1} w_*(x_j) \k_{\Ln}(x_j, x_i). 
\end{align*}
Because $\E[T \mid x_i]  = \E_{x'\sim q}[ w_*(x') \k_{\Ln}(x',x_i) ]=0$ for $\forall x$ and $| w(x') \k_{\Ln}(x,x') |\leq  L M_2 M_3$, $\forall x,x'\in \X$,
using Hoeffding's inequality, we have 
$$
\prob( w_i (\vx) <  0) = \prob(T> 1) \leq \exp( - \frac{n}{L^2 M_2^2 M_3^2}). 
$$

ii). Note that $S \overset{def}{=} \sum_i w_i (\vx) = S_1 + S_2$, 
\begin{align*}
\text{where} && S_1 = \frac{1}{n} \sum_{i=1}^n w_*(x_i), &&
S_2 = - \frac{2}{n^2} \sum_{i\in \D_0} \sum_{j\in \D_1} w_*(x_i) w_*(x_j) \k_\Ln(x_i, x_j), 
\end{align*} 
where the first term is the standard importance sampling weights and the second term comes from the control variate. It is easy to show that $\E[S_1] = 1$ and $\E[S_2] = 0$, and hence $\E[S]=1$.  
To prove the tail bound, note that for any $t_1 + t_2 = t$, $t_1, t_2 > 0$, we have 
\begin{align*}
\prob(S < 1 - t) 
& \leq \prob(S_1 < 1 - t_1) + \prob(S_2 \leq t_2) \\
& \leq \exp(- \frac{2 n t_1^2}{M_3^2}) + \exp( - \frac{4 n t_2^2}{ L^2 M_2^2 M_3^4} ), 
\end{align*}
where the bound for $S_2$ uses the Hoffeding's inequality for two-sample U statistics \citep[][Section 5b]{hoeffding1963probability}. We take $t_1 = \sqrt{2 t} / (L M_2 M_3 + \sqrt{2})$, we have
$$
\prob(S < 1 - t)  \leq 2 \exp( - \frac{4 n t^2} {L^2 M_3^2(M_2 M_3 + \sqrt{2}/L)^2}) \leq 2 \exp(- \frac{n t^2}{L^2 M_s}),
$$
where $M_s = M_3^2(M_2 M_3 + \sqrt{2})^2/4$ (we assume $L\geq 1$). 

\end{proof}

\begin{lem} \label{lem:boundwp}
Under Assumption~\ref{ass:kr}, we have 
$$
\E[\S(\{x_i, w_i^+(\vv x)\}, ~ p )]  \leq \frac{1}{4} \E[\S(\{x_i, w_i(\vv x) \}, ~ p) ]   
~+ ~  M_f (n+2) \exp(-\frac{n}{L^2 M_0})  ,
$$
where $M_f =  \trace(k_p(x,x')) M_2 $ and  $M_0 = \max(M_2^2 M_3, ~  M_3^2(M_2 M_3 + \sqrt{2})^2)$. 
\end{lem}
\begin{proof}
We use short notation $f(\vv w^+) = \S(\{x_i, w^+_i(\vv x)\}, ~ p )$ for convenience. We have 
\begin{align*}
|f(\vv w^+)| = |\sum_\ell \lambda_\ell (\sum_i w_i^+ \phi_\ell(x_i))^2| \leq  \trace(k_p(x,x')) M_2 \overset{\mathrm{def}}{=} M_f.
\end{align*}
Define $\mathcal E_n$ to be the event that all $w_i > 0$ and $\sum_i w_i \geq 1/2$, that is, $\mathcal E_n = \{ \sum_i w_i \geq 1/2, ~~ w_i \geq 0, ~\forall i\in [n]\}$. We have from Lemma\ref{lem:ineq} that 
$$
\prob( \mathcal{\bar{E}}_n)  \leq n \exp(-\frac{n}{L^2 M_2^2 M_3}) + 2 \exp(-\frac{n}{4 L^2 M_s}). 
$$
Note that under event $\mathcal E_n$, we have $\vv w = \vv w^+$. 
Therefore, 
\begin{align*}
\E[ f(\vv w^+ )]  
& = \E[f(\vv w^+) \mid \mathcal E_n]  \cdot \prob[\mathcal E_n] ~+ ~ \E[f(\vv w^+) \mid\mathcal{\bar{E}}_n]   \cdot \prob[\mathcal{\bar{E}}_n] \\
& \leq \E[f(\vv w^+) \mid \mathcal E_n]  \cdot \prob[\mathcal E_n] ~+ ~ M_f \cdot \prob[\mathcal{\bar{E}}_n] \\
& \leq  \frac{1}{4}  \E[f(\vv w) \mid \mathcal E_n]  \cdot \prob[\mathcal E_n] ~+ ~ M_f \cdot \prob[\mathcal{\bar{E}}_n] \\
& \leq \frac{1}{4} \E[f(\vv w) ]   ~+ ~ M_f \cdot \prob[\mathcal{\bar{E}}_n] \\
& \leq \frac{1}{4}  \E[f(\vv w)] ~+ ~  M_f \cdot \bigg [   n \exp(-\frac{n}{L^2 M_2^2 M_3})  +  2 \exp(- \frac{n }{4L^2 M_s}) \bigg ] \\
&  \leq \frac{1}{4}  \E[f(\vv w)]  ~+ ~  M_f (n+2) \exp(-\frac{n}{L^2 M_0})  
\end{align*}

\end{proof}

\section{Additional Empirical Results}
Here we show in Figure~\ref{fig:gmmv2} an additional empirical result when $p(x)$ is a  Gaussian mixture model shown in Figure~\ref{fig:gmmv2}(a) and 
 $\{x_i\}_{i=1}^n$ is generated by running $n$ independent chains of MALA for $10$ steps. 

\begin{figure}[htbp]
   \centering
   \begin{tabular}{ccccc}
 \hspace{-1em}   \raisebox{1.5em}{  \includegraphics[height=0.12\textwidth]{figures/contour.pdf} } \hspace{-2em}&
     \hspace{-1em}   \includegraphics[height=0.2\textwidth]{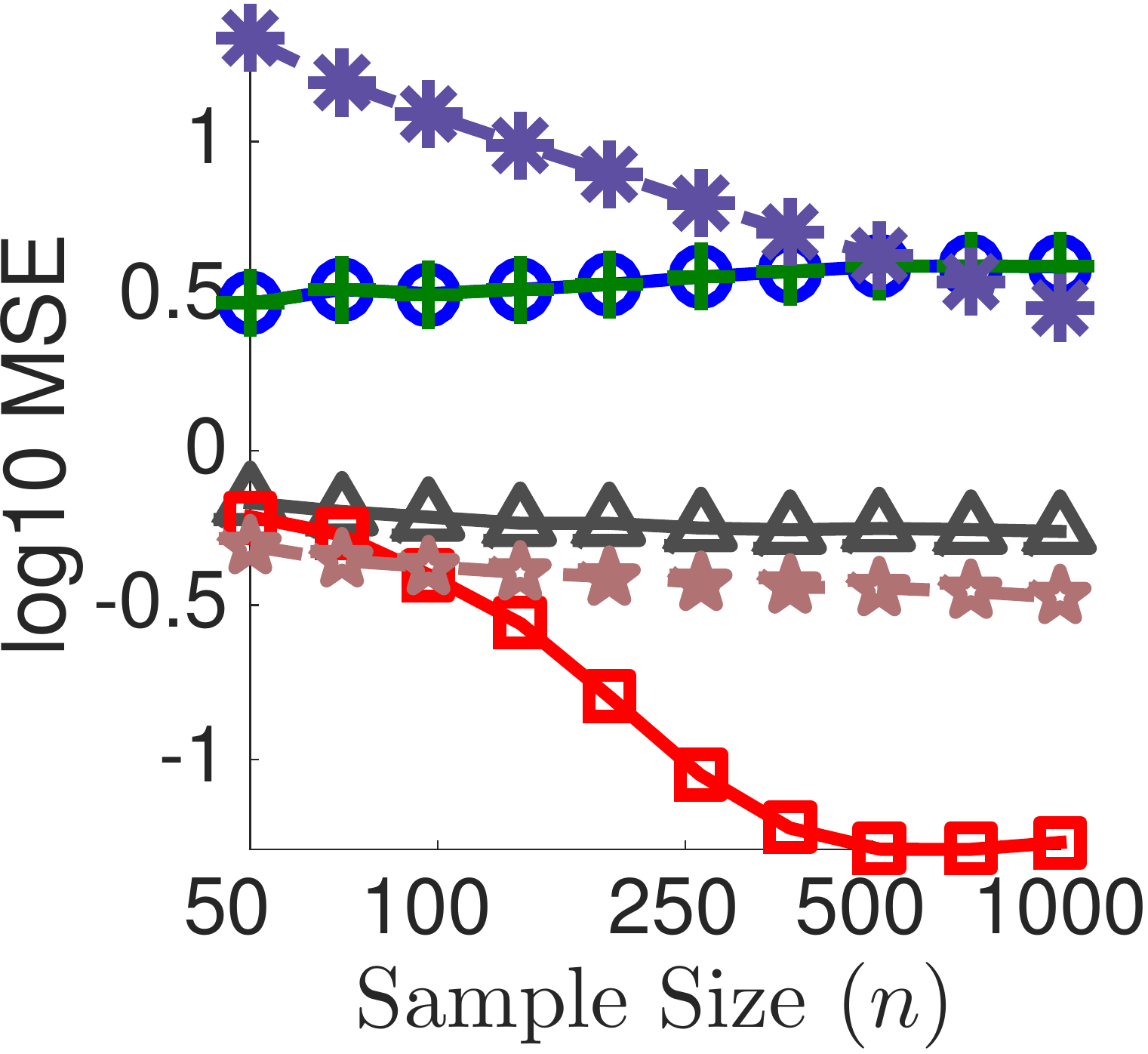} &
 \hspace{-1em}  \raisebox{0em}{  \includegraphics[height=0.2\textwidth, trim={1.28cm 0 0 0},clip]{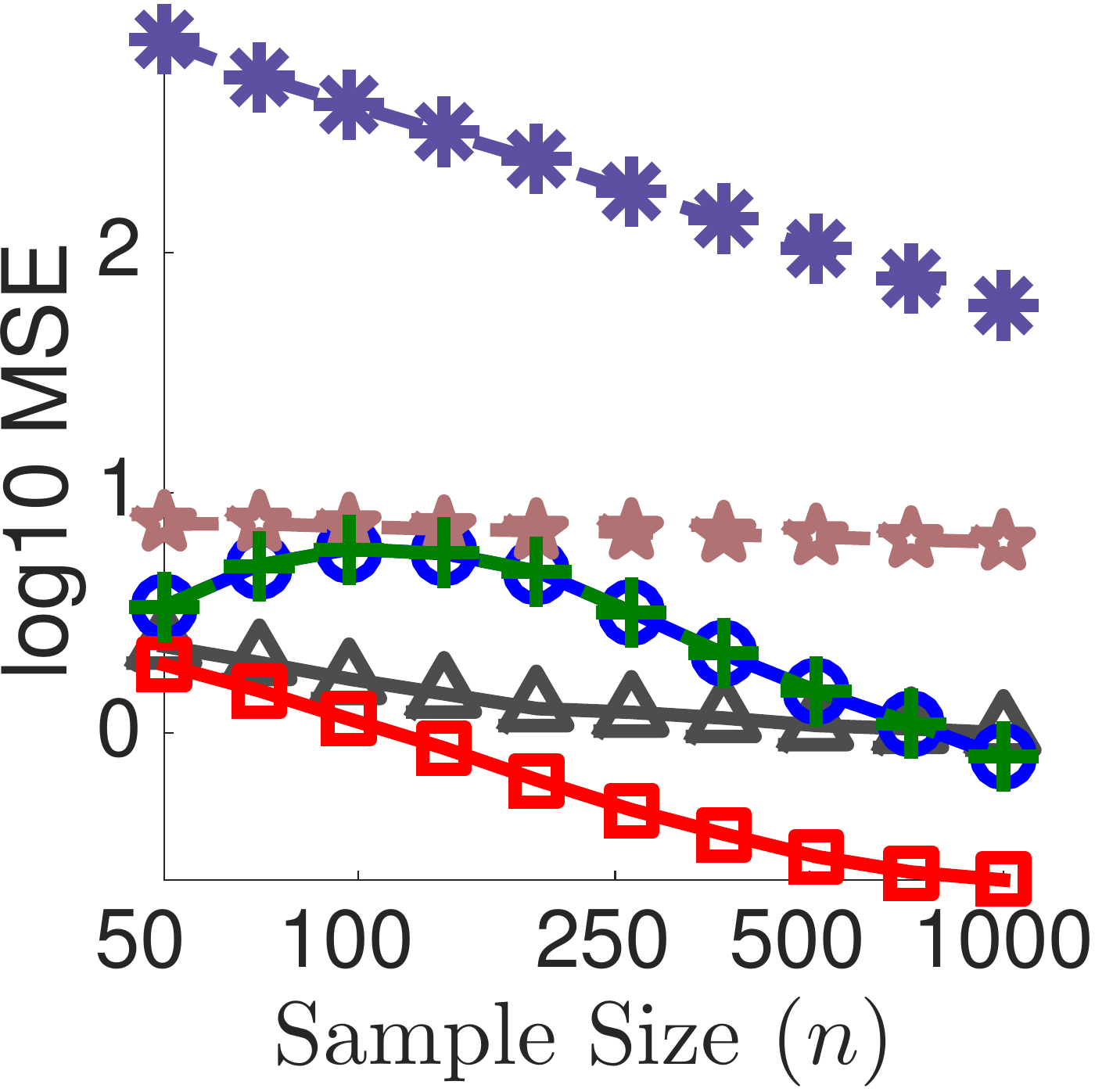}  }&
 \hspace{-1em}  \raisebox{0em}{  \includegraphics[height=0.2\textwidth, trim={1.4cm 0 0 0},clip]{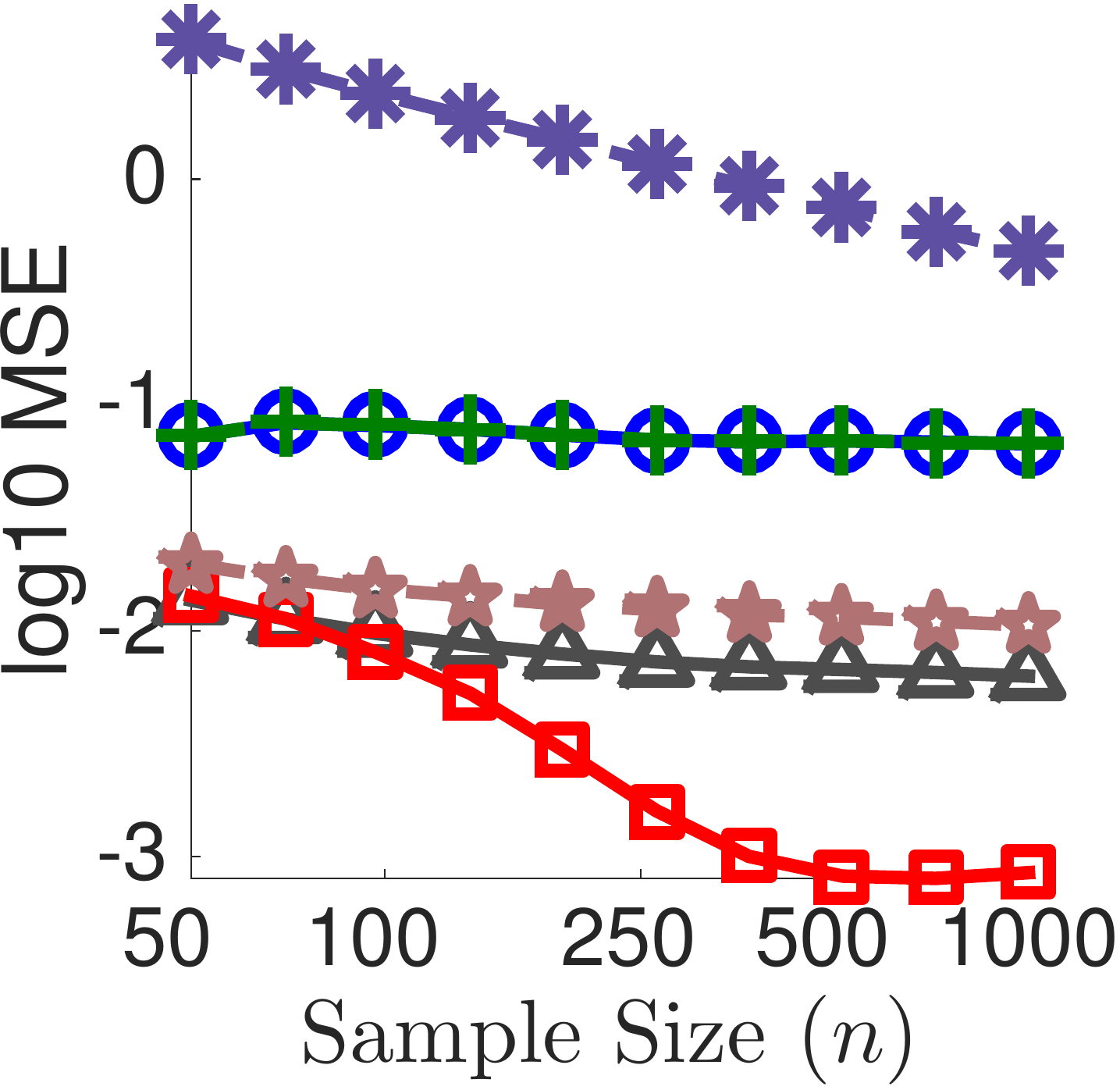}  }&
 \hspace{-2em}  \raisebox{4em}{  \includegraphics[height=0.1\textwidth]{figures/Leg_51steinISmala0d0039486_binprobitn1000d50wtrue1_50_0mu0invS0d01_NVec50_10_1000_sVec100.pdf}  } \\  
{ \small (a) $p(x)$} & {\small(b)  $\E(x)$} & {\small(c)  $\E(x^2)$} & {\small(d)  $\E(\cos(\omega x+b))$} &
   \end{tabular}
   \caption{Gaussian Mixture Example. (a) The contour of the distribution $p(x)$ that we use, and 
    $\{x_i\}_{i=1}^n$ is generated by running $n$ independent MALA for $10$ steps. 
   (b) - (c) The MSE of the different weighting schemes
   for estimating $\E(h(x))$, where $h(x)$ equals $x$, $x^2$, and $\cos(\omega x+b)$, respectively. For $h=\cos(\omega x+b)$ in (c), we draw $\omega\sim\normal(0,1)$ and $b \sim \textrm{Uniform}([0,2\pi])$ and average the MSE over $20$ random trials.}
   \label{fig:gmmv2}
\end{figure}

%

%
\onecolumn
\aistatstitle{Black-Box Importance Sampling}
\aistatsauthor{ Anonymous Author 1 \And Anonymous Author 2 \And Anonymous Author 3 }
\aistatsaddress{ Unknown Institution 1 \And Unknown Institution 2 \And Unknown Institution 3 } 

\bibliographystyle{unsrtnat}
\bibliography{bibrkhs_stein}
\end{document}